\documentclass[10pt]{article}
\usepackage[T1]{fontenc}
\usepackage[utf8]{inputenc}
\usepackage[letterpaper,left=1.5in,right=1.5in,top=1in,bottom=0.75in,headheight=14pt]{geometry}
\usepackage{times}
\usepackage{amsmath,amssymb,amsthm,mathtools}
\usepackage{booktabs,array,graphicx,enumitem}
\usepackage{tikz}
\usetikzlibrary{arrows.meta,positioning,calc}
\usepackage[round,authoryear]{natbib}
\usepackage{microtype}
\usepackage{hyperref}
\hypersetup{hidelinks,pdftitle={Block-Wise Differentiable Sinkhorn Attention: Tail-Refinement Gradients with a Single-Active-Dustbin Bridge},pdfauthor={Dylan B. Forde}}
\newtheorem{proposition}{Proposition}
\newtheorem{theorem}{Theorem}
\newtheorem{corollary}{Corollary}
\newtheorem{lemma}{Lemma}
\theoremstyle{remark}\newtheorem{remark}{Remark}
\newcommand{\one}{\mathbf 1}
\newcommand{\diag}{\operatorname{diag}}
\newcommand{\osc}{\operatorname{osc}}
\newcommand{\TV}{\mathrm{TV}}
\newcommand{\eps}{\varepsilon}
\newcommand{\surr}{\mathrm{surr}}
\newcommand{\full}{\mathrm{full}}
\title{\bfseries Block-Wise Differentiable Sinkhorn Attention:\ Tail-Refinement Gradients with a\ Single-Active-Dustbin Bridge}
\author{Dylan B. Forde\\Independent Researcher\\\texttt{forde.dylan@gmail.com}}
\date{}
\begin{document}
\maketitle
\begin{abstract}
We study finite reverse-mode differentiation of long-context Sinkhorn attention through a stopped-base, fixed-depth tail-refinement surrogate. For the reported $R=2$ implementation, four staircase plan factors lie in one row/column scaling orbit, so the exact score cotangent can be evaluated from one reference plan tile and vector modifiers with $O((T+R)LW)$ work and $O(L)$ auxiliary state for fixed head dimension and band width. The result is exact for the stated stopped-base surrogate, not full backpropagation through the stopped base solve. The extended theory shows that this is a specialization of an arbitrary finite same-score transport-orbit theorem, derives general-marginal half-step VJPs and the required centered-gauge corrections, and gives positive-block Hilbert certificates together with Dobrushin decay certificates for quotient cotangents under explicit positivity or mixing hypotheses. Synthetic tests check the centered surrogate and orbit identities; algebraic exactness is distinct from the finite-precision discrepancies reported below. Follow-up TPU experiments verify one-reference and direct-four/recompute complete arbitrary-cotangent Q/K/V VJPs through tested per-device context $L=262{,}144$ on an all-eight-device data-parallel v6e surface. A three-seed, 100-update Pfam screen favors trained FlashAttention on reconstruction in all observed pairs, while a prespecified direct-alignment-objective calibration selects no weight because within-five accuracy does not improve. These systems and task results establish scoped feasibility and limitations, not a universal latency ranking, dynamic-HBM advantage, or biological-alignment superiority.
\end{abstract}

\section{Introduction}
Pairwise alignment and structural matching remain central problems for sequence modeling. Classical dynamic-programming methods such as Smith--Waterman \citep{smith1981identification} provide optimal alignments but are not differentiable, so they cannot be inserted directly into end-to-end neural architectures. Entropic optimal transport (OT) offers a natural differentiable relaxation \citep{cuturi2013sinkhorn,peyre2019computational}: it replaces a discrete assignment with a smooth transport plan computed by Sinkhorn scaling.

For long-context models, however, two bottlenecks dominate. First, dense OT attention materializes $L\times L$ scores or plans, which is prohibitive at large $L$. Second, backward-pass design is difficult. Exact autodiff through many Sinkhorn steps is expensive and memory-hungry, while implicit differentiation can require large and unstable linear solves \citep{luise2018differential,bai2019deep}. The key question is therefore not only how to make OT attention differentiable, but how to make its differentiated object explicit, exact for the surrogate actually used, and compatible with accelerator kernels. Convergence of derivatives of discrete Sinkhorn--Knopp iterates to derivatives of the entropic OT solution is established by \citet{pauwels2023derivatives}. We do not claim that general convergence result: our exactness statement concerns the specified finite stopped-base surrogate, with finite-precision implementation error reported separately.

We address this with a stopped-base, fixed-depth \emph{tail-refinement surrogate}. After a stopped $T$-step balanced Sinkhorn solve, we append a short differentiable tail of depth $R$ and differentiate that surrogate exactly. The reported TPU path uses $R=2$. A direct implementation of the exact $R=2$ adjoint must evaluate four distinct staircase plan factors. The signature implementation observation of the paper is that these factors lie in a single row/column scaling orbit, so each block can be evaluated from one reference transport tile plus vector modifiers, avoiding separate plan-tile formation for the remaining three terms.

\paragraph{Single-active-dustbin bridge.}
The central factorization is balanced and fixed-support. The first extension carried in the unified manuscript is intentionally narrow: the current \texttt{dustbin\_block} path is not treated as a new OT solver, but as the same unit-target surrogate on an augmented state space with one active dustbin token per side. This keeps the implemented dustbin path inside the algebraic orbit of the main factorization, while leaving general gap-capacity and KL-unbalanced transport models to future work.

\paragraph{Related context.}
Our work is closest in spirit to recent efforts to make OT-style attention practical at accelerator scale, including FlashSinkhorn \citep{ye2026flashsinkhorn}. FlashSinkhorn is an IO-aware solver for streaming entropic OT updates; our contribution is the differentiated stopped-base backward object. A direct $R=2$ adjoint evaluates four staircase plan factors, whereas our transport-orbit calculus evaluates each block from one reference transport tile plus vector modifiers. The follow-up evidence now includes a literal four-resident comparator, a direct-four/recompute comparator, trained FlashAttention \citep{dao2022flashattention} on the fixed Pfam reconstruction screen, and the supported official FlashSinkhorn dense-forward interface. These surfaces answer different questions: official FlashSinkhorn does not expose the submitted hard-band arbitrary-cotangent Q/K/V backward, dense FlashAttention changes support, and local SplashAttention changes the operator from transport to softmax. We therefore report them without claiming a matched official-FlashSinkhorn backward result or wall-clock dominance. The approach is also complementary to structured-sparsity architectures such as Longformer \citep{beltagy2020longformer} and BigBird \citep{zaheer2020bigbird}.

The paper makes three main contributions:
\begin{itemize}[leftmargin=*]
\item an explicit balanced fixed-support tail-refinement surrogate together with its exact $R=2$ adjoint and one-reference-tile backward schedule;
\item a theorem-aligned single-active-dustbin augmentation showing that the current dustbin transport path is the same unit-target surrogate on an enlarged state space;
\item an empirical bridge from exactness validation to TPU-scale training, including a completed four-configuration Pfam screen, a promoted balanced run, matched backward-path diagnostics through tested context 262,144, and reviewer-motivated trained-baseline and direct-objective controls.
\end{itemize}

\section{Method}
\subsection{Forward Pass: Block-Wise Balanced Sinkhorn}
Let $Q,K,V\in\mathbb R^{L\times d}$ be query, key, and value tensors for sequence length $L$ and head dimension $d$, let $\eps>0$, and let $\Omega$ denote a fixed active support with band width $W$. Following standard attention notation, write the scaled-dot-product logits and entropic effective scores as
\[
 A_{ij}=\frac{Q_iK_j^\top}{\sqrt d},\qquad S_{ij}=\frac{A_{ij}}{\eps}.
\]
Thus the positive OT Gibbs kernel is $K_{ij}=\exp(S_{ij})$, equivalent to the usual $\exp(-C_{ij}/\eps)$ convention with $C=-A$ \citep{peyre2019computational}. We use $u,v$ for log-scalings; the positive Sinkhorn scaling vectors common in OT notation are $\exp(u),\exp(v)$. A bar, as in $\bar S$, denotes a reverse-mode cotangent. The symbols $\odot$ and $\oslash$ denote elementwise multiplication and division, respectively. The stopped base solve and the refinement tail both use the same masked balanced half-steps
\begin{align}
 u_i^{(t+1)}&=-\log\Bigl[\sum_{j:(i,j)\in\Omega}\exp(S_{ij}+v_j^{(t)})\Bigr],\label{eq:row-main}\\
 v_j^{(t+1)}&=-\log\Bigl[\sum_{i:(i,j)\in\Omega}\exp(S_{ij}+u_i^{(t+1)})\Bigr].\label{eq:col-main}
\end{align}
These are unit-target updates: each active row and active column target mass is one. When the active row and column counts are equal, this is the usual compatible balanced OT normalization. On rectangular masked Pfam examples, the implementation still uses this same unit-target Sinkhorn attention map; it should be read as the differentiable transport surrogate actually evaluated by the code, not as a compatible-marginal balanced OT problem with unequal total masses. The implementation computes these updates in tiles, in the same spirit as IO-aware attention kernels. For fixed band width $W$, a $T$-step base solve followed by an $R$-step tail has cost $O((T+R)LW)$ time, $O(Ld)$ input storage, and $O(L)$ additional asymptotic working state.

\begin{remark}[Fixed-Budget $\eps$-Scaling]
The implementation optionally uses a short continuation schedule $\eps_1\ge\cdots\ge\eps_{N_\eps}=\eps$ during the stopped base solve. If the stage iteration counts satisfy $\sum_{s=1}^{N_\eps}T_s=T$, the asymptotic cost remains $O(TLW)$ time and $O(L)$ additional working state for fixed head dimension.
\end{remark}

\subsection{Balanced Tail Refinement and the \texorpdfstring{$R=2$}{R=2} Staircase}
Let $(u^{(0)},v^{(0)})$ be the stopped potentials produced by the base $T$-step solve. The differentiated object is not the full $T$-step map; it is the fixed-depth surrogate
\[
 \widetilde O^{[R]}(Q,K,V;u^{(0)},v^{(0)})=\mathcal A(Q,K,V,u^{(R)},v^{(R)}),
\]
where the refinement iterates are recomputed from the stopped base pair. The implementation stores centered dual representatives for numerical stability, while the main algebra is written in ungauged coordinates. Appendix~\ref{app:proofs} gives the corresponding centered gauge ledger: same-time plans are gauge-invariant, while mixed-time centered plans require explicit scalar gauge factors when translated back to ungauged coordinates. The terminal transport application is
\[
 \mathcal A(Q,K,V,u,v)_i=\sum_{j:(i,j)\in\Omega}P_{ij}(u,v)V_j,\qquad
 P_{ij}(u,v)=\exp(S_{ij}+u_i+v_j).
\]
There is no extra row renormalization in the training output. Reconstruction uses this raw transport output, sparse CE logs the target-entry transport mass, and row-normalized barycenters are used only for alignment diagnostics and visualization.

Define staircase plans
\[
 P_{ij}^{(p,q)}=\begin{cases}\exp(S_{ij}+u_i^{(p)}+v_j^{(q)}),&(i,j)\in\Omega,\\0,&(i,j)\notin\Omega.\end{cases}
\]
Let $\mathcal L$ be a scalar loss depending on $\widetilde O^{[R]}$, let $G=\partial\mathcal L/\partial\widetilde O^{[R]}$, and define
\[
 Z_{ij}=\langle G_i,V_j\rangle,\quad
 g_u=(P^{(R,R)}\odot Z)\one,\quad g_v=(P^{(R,R)}\odot Z)^\top\one.
\]
\begin{proposition}[Exact Fixed-Depth Tail Adjoint]\label{prop:adjoint}
The cotangents of the refinement duals satisfy
\begin{align}
 \bar v^{(R)}&=g_v,\label{eq:vr}\\
 \bar u^{(R)}&=g_u-P^{(R,R)}\bar v^{(R)},\label{eq:ur}\\
 \bar v^{(t-1)}&=-(P^{(t,t-1)})^\top\bar u^{(t)},&&t=R,R-1,\ldots,1,\label{eq:vreverse}\\
 \bar u^{(t-1)}&=-P^{(t-1,t-1)}\bar v^{(t-1)},&&t=R,R-1,\ldots,2,\label{eq:ureverse}
\end{align}
where the final line stops at $t=2$ because the base state is stop-gradient. The score-space cotangent is
\begin{equation}\label{eq:score}
 \bar S=P^{(R,R)}\odot Z-\sum_{t=1}^{R}\left[
 P^{(t,t)}\odot(\one\bar v^{(t)\top})+
 P^{(t,t-1)}\odot(\bar u^{(t)}\one^\top)\right].
\end{equation}
\end{proposition}
For the reported TPU kernel we use $R=2$, which yields the staircase
$\bigl(P^{(2,2)},P^{(2,1)},P^{(1,1)},P^{(1,0)}\bigr)$.

\begin{proposition}[One-Reference-Tile $R=2$ Backward Schedule]\label{prop:one-ref}
For any staircase pair $(p,q)\in\{(2,2),(2,1),(1,1),(1,0)\}$,
\[
 P^{(p,q)}=P^{(2,2)}\odot\left[\exp(u^{(p)}-u^{(2)})\bigl(\exp(v^{(q)}-v^{(2)})\bigr)^\top\right].
\]
Define
\[
 \alpha_i=e^{u_i^{(1)}-u_i^{(2)}},\qquad
 \beta_j=e^{v_j^{(1)}-v_j^{(2)}},\qquad
 \delta_j=e^{v_j^{(0)}-v_j^{(2)}}.
\]
Then on the active support the exact $R=2$ score-space VJP is
\[
 \bar S_{ij}=P_{ij}^{(2,2)}\left[Z_{ij}-\bar v_j^{(2)}-\beta_j\bar u_i^{(2)}-
 \alpha_i\beta_j\bar v_j^{(1)}-\alpha_i\delta_j\bar u_i^{(1)}\right].
\]
Hence, for a $B\times B$ score tile, the exact backward can be scheduled with one resident reference plan tile $P_{I,J}^{(2,2)}$, the corresponding $O(B)$ vector slices, and the $O(Bd)$ activation/cotangent slices needed to form $S_{I,J}$ and $Z_{I,J}$. The remaining staircase factors are never formed as resident plan tiles. For fixed block size, head dimension, band width, $T$, and $R=2$, the schedule uses $O(Ld)$ input storage, $O(L)$ global vector state, and $O(LW)$ tile arithmetic for the tail backward.
\end{proposition}

\begin{remark}[Centered gauge ledger]
The displayed formula is the ungauged statement. If the stored centered representatives satisfy
\[
 \widetilde u^{(t)}=u^{(t)}-c_t\one,\qquad \widetilde v^{(t)}=v^{(t)}+c_t\one,
\]
where $c_0=0$ for the common centered base pair, then same-time plans are unchanged but mixed-time plans are not. With
\[
 \widetilde\alpha_i=e^{\widetilde u_i^{(1)}-\widetilde u_i^{(2)}},\quad
 \widetilde\beta_j=e^{\widetilde v_j^{(1)}-\widetilde v_j^{(2)}},\quad
 \widetilde\delta_j=e^{\widetilde v_j^{(0)}-\widetilde v_j^{(2)}},
\]
the ungauged mixed factors recovered from the centered ledger are
\begin{align*}
 P_{ij}^{(2,1)}&=e^{c_2-c_1}P_{ij}^{(2,2)}\widetilde\beta_j,&
 P_{ij}^{(1,1)}&=P_{ij}^{(2,2)}\widetilde\alpha_i\widetilde\beta_j,\\
 P_{ij}^{(1,0)}&=e^{c_1-c_0}P_{ij}^{(2,2)}\widetilde\alpha_i\widetilde\delta_j.
\end{align*}
Appendix~\ref{app:proofs} records this explicitly because the mixed-time factors are not gauge-invariant.
\end{remark}
For $R=2$, equation~\eqref{eq:score} expands to
\begin{align}\label{eq:r2-expanded}
\bar S={}&P^{(2,2)}\odot Z-P^{(2,2)}\odot(\one\bar v^{(2)\top})-P^{(2,1)}\odot(\bar u^{(2)}\one^\top)\nonumber\\
&-P^{(1,1)}\odot(\one\bar v^{(1)\top})-P^{(1,0)}\odot(\bar u^{(1)}\one^\top).
\end{align}
The reconstruction identity above then replaces separate formation of the last three plan tiles with row and column rescalings of $P^{(2,2)}$.
\begin{figure}[htbp]\centering
\includegraphics[width=\linewidth]{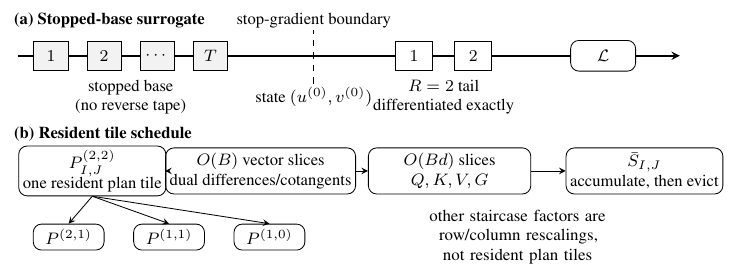}
\caption{Stopped-base surrogate and single-reference-tile $R=2$ backward schedule. Top: a $T$-step stopped solve initializes two exactly differentiated tail refinements. Bottom: each score tile is accumulated from one resident $P_{I,J}^{(2,2)}$ tile and vector/activation slices; the other staircase factors are recovered by row/column rescaling rather than materialized.}\label{fig:schedule}
\end{figure}
Figure~\ref{fig:schedule} summarizes the stop-gradient boundary and the tile-level memory schedule. The schedule is the signature implementation observation in the manuscript: it converts the explicit four-factor $R=2$ backward into evaluation from one reference plan tile plus cheap row and column modifiers, which is why the method maps cleanly to streaming accelerator kernels.

\paragraph{Why $R=2$.}
The multiplicative orbit identity itself is not unique to $R=2$. Relative to any chosen base staircase plan, every other staircase plan factors by the corresponding dual differences in the same way. What is special about $R=2$ is that the exact adjoint closes into the compact four-factor formula above, so the identity yields a concrete one-reference-tile schedule rather than a larger family of rescaled plans.

\subsection{Single-Active-Dustbin Bridge}\label{sec:dustbin}
The current dustbin TPU path does not use a new OT solver. Instead, it augments the unit-target problem. Let $B_{\rm db}=128$ denote the dustbin block size used in the implementation. Define augmented tensors
\[
 \widetilde Q,\widetilde K,\widetilde V\in\mathbb R^{(L+B_{\rm db})\times d}
\]
by appending one learned dustbin token per side together with $B_{\rm db}-1$ inactive filler tokens. The masks are augmented analogously: one appended token is active, the rest are inactive fillers. The implementation then applies the same balanced tail-refinement solver to this augmented system, with effective band width $\widetilde W=\max(W,L)$.

The current reported implementation is exactly the unit-target surrogate on the active support induced by this widened generic band mask after inactive fillers are excluded; denote that support by $\widetilde\Omega_{\rm impl}$. Write $\Omega_{\rm band}:=\Omega$ for the original band support, and let $I_q,I_k$ denote the active original query and key indices. A hypothetical sparse-spoke implementation would instead use
\[
\widetilde\Omega_{\rm spoke}=\Omega_{\rm band}\cup\{(i,d_k):i\in I_q\}\cup\{(d_q,j):j\in I_k\}\cup\{(d_q,d_k)\},
\]
where $d_q,d_k$ are the active appended dustbin row and column; that support would have $|\Omega_{\rm band}|+O(L)$ active edges. The inactive fillers carry zero mask mass. Because the current reported code path realizes the bridge through the generic widened band mask $\widetilde W=\max(W,L)$, its arithmetic should be accounted as $O(L\widetilde W)$, which becomes dense $O(L^2)$ when $W<L$. We do not claim the current dustbin path has the same long-context complexity as the balanced banded path in that regime.

\begin{proposition}[Dustbin Augmentation Bridge]\label{prop:dustbin}
Let $(\widetilde Q,\widetilde K,\widetilde V,\widetilde\Omega_{\rm impl})$ denote the dustbin-augmented tensors and implemented active support constructed above. Then the current \texttt{dustbin\_block} path is exactly the same unit-target fixed-support tail-refinement surrogate as in Section~2.2, but applied on the augmented state space.
\end{proposition}
\begin{corollary}[Exact $R=2$ Adjoint on the Augmented State Space]\label{cor:dustbin}
Apply Propositions~\ref{prop:adjoint} and~\ref{prop:one-ref} to the augmented system. Then the current dustbin-augmented transport path inherits the exact fixed-depth adjoint and the same one-reference-tile schedule on the augmented support.
\end{corollary}
This is the first dustbin extension carried into the unified paper because it already matches the cloud path. It is a bridge theorem, not a departure from the central balanced factorization. With one active dustbin token per side, the dustbin marginal capacity is one unit on each side; arbitrary gap capacity would require enlarged dustbin marginals, multiple active dustbin tokens, or a genuinely unbalanced model.

\subsection{A Local Surrogate-Bias Bound}
The fixed-depth backward is exact for the surrogate it differentiates. A separate question is how far that surrogate gradient can be from the gradient through the full stopped-base computation. Let
\[
 z=(u,v),\qquad \Phi_x(z)=z^+,
\]
be the centered one-step refinement map at differentiable parameters $x$, and let $\mathcal A_x(z)$ denote the transport application map. Let $\ell_x(z)$ denote the scalar terminal loss applied to $\mathcal A_x(z)$, and take the initialization $z_0$ to be independent of $x$. Write
\[
 z_T(x)=\Phi_x^T(z_0),\qquad \Psi_x^R(z)=\ell_x(\Phi_x^R(z)).
\]
Then the full fixed-depth gradient is $g_{\full}^{T,R}(x)=\nabla_x\bigl(\Psi_x^R(z_T(x))\bigr)$, while the stop-gradient surrogate uses
\[
 g_{\surr}^{T,R}(x)=\left[D_x\Psi_x^R(z)\big|_{z=z_T(x)}\right]^\top.
\]
\begin{theorem}[Local Surrogate-Bias Bound]\label{thm:bias}
Assume fixed support, $C^1$ dependence of $\Phi_x$ and $\ell_x$ on $(x,z)$, and constants $\rho\in[0,1)$, $L_\Phi>0$, and $L_\ell>0$ such that
\[
 \|D_z\Phi_x(z)\|\le\rho,\qquad \|D_x\Phi_x(z)\|\le L_\Phi,\qquad \|\nabla_z\ell_x(z)\|\le L_\ell
\]
on a forward-invariant neighborhood. Then
\[
 g_{\full}^{T,R}(x)-g_{\surr}^{T,R}(x)=D_xz_T(x)^\top\nabla_z\Psi_x^R(z_T(x)),
\]
and
\[
 \|g_{\full}^{T,R}(x)-g_{\surr}^{T,R}(x)\|\le\frac{L_\ell L_\Phi}{1-\rho}\rho^R.
\]
\end{theorem}
The point of this bound is modest but useful: once the support is fixed and the refinement map is locally contractive, the stop-gradient bias decays geometrically in the tail depth $R$. We also record the finite-instance version used by the validation harness.

\begin{proposition}[Projective Sinkhorn Contraction Certificate]\label{prop:hilbert}
Consider a strictly positive active block of the balanced or dustbin-augmented support, and let $K_{ij}=\exp(S_{ij})$ be the positive kernel on that block, using the scaled score convention from Section~2.1. For positive marginals $a,b$ and positive column scalings $w,w'$, define the two-half-step Sinkhorn scaling map
\[
 \mathcal T(w)=b\oslash\bigl(K^\top(a\oslash(K w))\bigr),
\]
where our normalized update has $a=b=\one$. Let $d_H$ denote Hilbert's projective metric, let $\Delta(\mathcal B)$ be the projective diameter of a positive linear map $\mathcal B$, and define $\osc(h)=\max_i h_i-\min_i h_i$. Then
\[
 d_H(\mathcal T(w),\mathcal T(w'))\le\rho_Hd_H(w,w'),\quad
 \rho_H=\tanh\!\left(\frac{\Delta(K)}4\right)\tanh\!\left(\frac{\Delta(K^\top)}4\right)<1.
\]
Equivalently, in log scalings, with $\Phi_{\rm col}(v)=\log\mathcal T(e^v)$ modulo additive gauge, the centered half-step pair is a contraction in oscillation seminorm:
\[
 \osc(\Phi_{\rm col}(v)-\Phi_{\rm col}(v'))\le\rho_H\osc(v-v').
\]
If $\osc_\Omega(S):=\max_{(i,j)}S_{ij}-\min_{(i,j)}S_{ij}$ on the active block, then a computable sufficient bound is
\[
 \rho_H\le\tanh^2\!\left(\frac{\osc_\Omega(S)}2\right)<1.
\]
Equivalently, for the scaled-dot-product attention logits $A$ with $S=A/\eps$ and $\osc_\Omega(A)=\max A-\min A$ on the active block, this reads
$\rho_H\le\tanh^2\bigl(\osc_\Omega(A)/(2\eps)\bigr)$.
\end{proposition}
This proposition is a support-local certificate, not a claim about masked zeros. The Birkhoff--Hopf argument \citep{birkhoff1957extensions,bushell1973hilbert} applies on strictly positive active blocks after masked entries and inactive dustbin fillers have been excluded. The dustbin bridge preserves the statement because the active dustbin token is part of the augmented positive block, while the inactive filler tokens are outside the cone on which the contraction is asserted.

\begin{proposition}[A Posteriori Bias Certificate]\label{prop:posterior}
Let $F_R(x,z)$ denote the scalar tail objective obtained by applying the terminal loss to $\mathcal A_x(\Phi_x^R(z))$, and define the measured base-state cotangent
\[
 \eta_R(x)=\nabla_zF_R(x,z)\big|_{z=z_T(x)}.
\]
Then the full-gradient remainder omitted by the stopped-base surrogate is exactly
\[
\Delta g_R:=\nabla_xF_R(x,z_T(x))-
\left[D_xF_R(x,z)\big|_{z=z_T(x)}\right]^\top
=D_xz_T(x)^\top\eta_R(x).
\]
Consequently, for any compatible norm,
$\|\Delta g_R\|\le\|D_xz_T(x)^\top\|\,\|\eta_R(x)\|$.
\end{proposition}
This certificate is operational: $\eta_R$ is the base cotangent produced by the same fixed-depth tail reverse pass before it is discarded by the stop-gradient rule, and $D_xz_T(x)^\top\eta_R$ is a single VJP through the base Sinkhorn solve.

\begin{corollary}[Certified Tail-Depth Selection]\label{cor:selection}
Fix a tolerance $\tau_{\rm cert}>0$ and a maximum tail depth $R_{\max}$. For each $R\le R_{\max}$, let $e_R(x)=D_xz_T(x)^\top\eta_R(x)$. Let $\operatorname{Cost}(R)$ be any nondecreasing implementation cost over the candidate depths $\{0,\ldots,R_{\max}\}$. If the implementation chooses the first $R$ such that $\|e_R(x)\|\le\tau_{\rm cert}$ and returns the stopped-base gradient at that depth, then the returned gradient differs from full BPTT through the same $T$-step base solve by at most $\tau_{\rm cert}$. Moreover, this first feasible $R$ solves the discrete certified optimization problem
\[
 \min_{0\le R\le R_{\max}}\operatorname{Cost}(R)\quad\text{subject to}\quad\|e_R(x)\|\le\tau_{\rm cert}.
\]
If an upper bound $\kappa_T(x)\ge\|D_xz_T(x)^\top\|$ is available, the same guarantee and optimality statement hold with the sufficient constraint $\kappa_T(x)\|\eta_R(x)\|\le\tau_{\rm cert}$.
\end{corollary}
This turns tail-depth choice into a small certified resource-allocation problem rather than a purely heuristic hyperparameter. Instead of solving an adjoint linear system, one can either use a fixed depth such as $R=2$ when the certificate has already been calibrated, or periodically run the certificate VJP to verify that the stopped-base omission remains below a chosen tolerance.

\begin{proposition}[Transport-Orbit Calculus]\label{prop:orbit}
Consider any finite collection of transport plans $P^{(p,q)}$ generated from the same masked score tensor $S$ and the same active support, including the plans used by the $R=2$ staircase, the $T$ stopped base reverse pass, the $R$ differentiable tail reverse pass, and the certificate VJP $D_xz_T^\top\eta_R$. This fixed-score condition is essential: under an $\eps$-scaling continuation schedule, the orbit statement applies stagewise to plans using the same effective score $S^{(s)}=A/\eps_s$, not across different temperature stages. Fix any reference plan $P^{\rm ref}=P^{(p_{\rm ref},q_{\rm ref})}$ in this collection and define
\[
 r^{(p)}=\exp(u^{(p)}-u^{(p_{\rm ref})}),\qquad s^{(q)}=\exp(v^{(q)}-v^{(q_{\rm ref})}).
\]
Then every plan in the collection lies in the row/column scaling orbit of $P^{\rm ref}$:
\[
 P^{(p,q)}=\diag(r^{(p)})P^{\rm ref}\diag(s^{(q)}).
\]
Equivalently, for any finite sum of plan-weighted score cotangent terms
$\Xi=\sum_{\ell=1}^{N}P^{(p_\ell,q_\ell)}\odot H_\ell$, one has
\[
 \Xi=P^{\rm ref}\odot\sum_{\ell=1}^N
 \bigl(r^{(p_\ell)}(s^{(q_\ell)})^\top\bigr)\odot H_\ell.
\]
Consequently, for fixed $T,R$, band width $W$, and head dimension $d$, the stopped gradient, the certified omitted-gradient VJP, and the certified tail-depth test can all be evaluated blockwise with $O((T+R)LW)$ arithmetic, $O(Ld)$ input storage, and $O((T+R)L)$ vector state, while materializing only one transport tile at a time. The compiler abstraction is therefore not a special four-plan trick: it is accumulation over a single transport orbit.
\end{proposition}

\section{Validation and Experiments}
\subsection{Exactness Validation of the Masked Surrogate}
We first separate surrogate correctness from systems scale. On synthetic masked problems for which exact autodiff through the same centered surrogate remains tractable, we compare the optimized $R=2$ kernel against a pure-JAX reference. Table~\ref{tab:exactness} reports representative maximum absolute gradient deviations under the default validation setting: float32 arithmetic, $\eps=1$, no $\eps$-scaling continuation, band width 256, $R=2$, centered potentials, and masked inactive entries zeroed.
\begin{table}[htbp]\centering\small
\begin{tabular}{@{}rrrrr@{}}\toprule
Length $L$ & $\max|\Delta Q|$ & $\max|\Delta K|$ & $\max|\Delta V|$ & Max rel. $\ell_2$\\\midrule
512 & $1.05\times10^{-5}$ & $1.73\times10^{-6}$ & $1.06\times10^{-9}$ & $5.78\times10^{-2}$\\
1024 & $5.70\times10^{-6}$ & $8.13\times10^{-7}$ & $2.55\times10^{-10}$ & $5.76\times10^{-2}$\\
2048 & $5.50\times10^{-6}$ & $9.17\times10^{-7}$ & $2.29\times10^{-10}$ & $5.05\times10^{-2}$\\\bottomrule
\end{tabular}
\caption{Gradient deviations between the optimized masked $R=2$ kernel and exact autodiff through the same centered surrogate. The final column is the worst relative $\ell_2$ error among $(\partial Q,\partial K,\partial V)$; the corresponding worst absolute $\ell_2$ errors are $6.59\times10^{-5}$, $4.11\times10^{-5}$, and $2.62\times10^{-5}$.}\label{tab:exactness}
\end{table}
The table is a finite-precision numerical comparison, not a proof of implemented gradient equivalence or a full numerical-analysis study. The relative $\ell_2$ errors are largest for the small-norm $Q$ gradients, whose reference $\ell_2$ norms are $1.14\times10^{-3}$, $7.14\times10^{-4}$, and $5.18\times10^{-4}$ in the three rows; the absolute $\ell_2$ errors remain below $6.6\times10^{-5}$ and the output relative $\ell_2$ errors are below $1.8\times10^{-7}$. Small reference norms explain the scale of the relative ratios, but do not identify the numerical source of the discrepancy. In particular, these measurements do not establish uniformly high relative gradient accuracy. The exact-arithmetic adjoint theorem and this approximately 5--6\% relative discrepancy are distinct claims; no new precision study or corrected measurement is asserted in this revision.

Additional forward-consistency checks run through lengths 512, 1024, 2048, 4096, 8192, and 16384, with extended spot checks at 32768, 65536, and 131072. We treat the larger checks as forward-consistency evidence rather than as exact-gradient checks.

To calibrate the local surrogate-bias discussion, we also compared tail-refinement gradients against full BPTT through the stopped-base solve on the same synthetic masked setup at $L=512$, $d=8$, $T=15$, $W=256$, and $\eps=1$. Over seeds 0, 1, 2, the worst-case gradient discrepancy $\max\{|\Delta Q|,|\Delta K|,|\Delta V|\}$ was $1.49\times10^{-4}$ for $R=0$, $1.63\times10^{-5}$ for $R=1$, $1.42\times10^{-5}$ for $R=2$, and $1.42\times10^{-5}$ for $R=4$. In this local exact-gradient calibration, the main empirical drop therefore occurs between zero tail refinement and the first one or two refinement steps, while $R=2$ and $R=4$ are already effectively tied at the scale of the masked validation harness.

Table~\ref{tab:biascert} reports the a posteriori certificate from Proposition~\ref{prop:posterior} on a smaller full-BPTT setting where the base-solve VJP can be checked directly ($L=128$, $d=8$, $T=15$, $W=128$, seeds 0, 1, 2). The certificate reconstructs the full-gradient remainder to numerical precision and shows rapid decay of the measured base-state cotangent.
\begin{table}[htbp]\centering\small
\begin{tabular}{@{}rrrr@{}}\toprule
Tail depth $R$ & $\max\|\eta_R\|_2$ & $\max|\Delta g_R|$ & certificate residual\\\midrule
0 & $5.89\times10^{-3}$ & $6.29\times10^{-4}$ & $2.47\times10^{-10}$\\
1 & $5.40\times10^{-4}$ & $5.93\times10^{-5}$ & $3.64\times10^{-11}$\\
2 & $5.47\times10^{-5}$ & $5.33\times10^{-6}$ & $2.84\times10^{-11}$\\
4 & $9.03\times10^{-7}$ & $9.99\times10^{-8}$ & $2.74\times10^{-11}$\\\bottomrule
\end{tabular}
\caption{A posteriori surrogate-bias certificate. The residual is the maximum absolute difference between the actual full-vs-stopped gradient gap and the certified base-solve VJP $D_xz_T^\top\eta_R$.}\label{tab:biascert}
\end{table}
Using the certified selector from Corollary~\ref{cor:selection} with max-absolute tolerance $10^{-5}$ selects $R=2$ on all three seeds in this calibration. Since the per-depth tail cost is monotone in $R$, this is the cheapest certified depth among the tested candidates and gives a certificate-level justification for using $R=2$ as the reported TPU depth.

We also validate the transport-orbit calculus from Proposition~\ref{prop:orbit} directly over all $2(T+R)$ base and tail half-step plans at $T=15$, $R=2$, $L=128$. Reconstructing every plan from the final reference plan gives maximum log-domain reconstruction error $1.91\times10^{-6}$. Appendix~\ref{app:protocol} reports two additional diagnostics. The controlled direct-four-plan score-adjoint diagnostic reduces logical active plan-factor storage by exactly $4\times$ and gives 1.29--2.20$\times$ post-compilation CPU speedups with maximum score-cotangent discrepancy at most $4.77\times10^{-7}$; it is still a dense-JAX CPU materialization check, not a production TPU kernel benchmark. The same appendix also reports projective contraction certificates computed from the promoted Pfam checkpoint.

\subsection{Short Pfam Screen on TPU v6e-8}
The first empirical bridge is the short Pfam screen on TPU v6e-8. All four core configurations complete end-to-end:
\begin{center}\small\texttt{k2\_eps1, k4\_eps1, muot\_dustbin\_mono, muot\_dustbin\_hybrid}.\end{center}
Each uses 50,000 training pairs, $T=15$ base iterations, band width 1024, and global batch size 64. Table~\ref{tab:screen} reports the step-50 screen snapshot.
\begin{table}[htbp]\centering\small
\begin{tabular}{@{}lrrrr@{}}\toprule
Config & Step 50 Loss & Step 50 CE & Step 50 Mono & Throughput (ex/s)\\\midrule
\texttt{k2\_eps1} & 3.5778 & 5.3223 & 0.3353 & 8.46\\
\texttt{k4\_eps1} & 3.5779 & 5.3223 & 0.3353 & 8.38\\
\texttt{muot\_dustbin\_mono} & 3.9030 & 5.3258 & 0.4155 & 8.41\\
\texttt{muot\_dustbin\_hybrid} & 4.6250 & 5.3244 & 0.4162 & 8.45\\\bottomrule
\end{tabular}
\caption{Short TPU v6e-8 Pfam screen across the four core balanced and dustbin-augmented configurations.}\label{tab:screen}
\end{table}
Two immediate conclusions follow. First, on the balanced path we observe no meaningful difference between $R=2$ and $R=4$ at this short-screen budget, while $R=2$ remains slightly faster; this is a short smoke screen rather than evidence that the runs have fully converged. Second, the single-active-dustbin bridge is operational in the same TPU training harness, with the widened-support cost caveat from Section~\ref{sec:dustbin}.

\subsection{Promoted Balanced Run}
We then promoted the balanced \texttt{k2\_eps1} configuration to the longer-budget TPU training path. The full objective details are given in the protocol notes in the appendix; for the interpretation of the main text, the important point is that this promotion is a trainability and systems-stability test of the factorization-backed balanced $R=2$ path. The run reached step 1437, repeatedly persisted checkpoints at steps 250, 500, 750, 1000, and 1250, and sustained approximately 8.5 examples per second after startup warmup. The initial throughput at step 0 was 2.67 examples per second, after which the run stabilized near 8.5 examples per second for the remainder of the budget.
\begin{table}[htbp]\centering\small
\begin{tabular}{@{}lrrrrr@{}}\toprule
Checkpoint / reference & Rec. & Sparse CE & Bary. MAE & Bary. $\le5$ & Mono.\\\midrule
Step 0 & 5.5730 & 5.5335 & 65.1369 & 0.0406 & 0.4309\\
Step 1437 & 2.0539 & 5.3046 & 64.5921 & 0.0403 & 0.5003\\
Diag. ref. (step-0 $V$) & 5.6320 & 17.8155 & 1.6367 & 0.9372 & 0.0000\\
Diag. ref. (step-1437 $V$) & 2.0142 & 17.8155 & 1.6367 & 0.9372 & 0.0000\\\bottomrule
\end{tabular}
\caption{Held-out Pfam test-shard evaluation on 100 examples. ``Bary.'' reports row-normalized transport-barycenter error against supervised target columns. The diagonal references use a length-scaled monotone point-mass alignment; their sparse CE uses a $10^{-12}$ off-diagonal floor and is not directly comparable to the soft Sinkhorn CE.}\label{tab:promoted}
\end{table}
Example-bootstrap 95\% intervals are [5.03, 6.06] and [1.78, 2.30] for reconstruction at steps 0 and 1437, and [5.49, 5.58] and [5.26, 5.35] for sparse CE. For the step-1437 diagonal reference, the intervals are [1.73, 2.31] for reconstruction, [1.33, 1.96] for barycenter MAE, and [0.90, 0.97] for barycenter within five columns. An untrained same-projection flash-attention reference gives reconstruction 5.5753 and sparse CE 5.5383, matching the step-0 Sinkhorn checkpoint; this is an initialization sanity reference, not a trained baseline. The promoted run is therefore evidence of stability, checkpointability, and reconstruction improvement for the factorization-backed path. The CE change is diagnostic, target-barycenter alignment metrics do not materially improve, barycenter-within-five and monotonicity do not improve, and the diagonal reference remains much stronger on alignment metrics. This experiment therefore does not establish an alignment-task advantage.

\subsection{Reviewer-Motivated Backward and Task Controls}
We ran a terminal all-eight-device data-parallel TPU v6e follow-up with one full-length example per device and no sequence sharding. The explicit complete-step executable returns the output and arbitrary-cotangent Q/K/V VJP for the stated $T=15$, $R=2$, $W=1024$ stopped-base surrogate. Both the one-reference and direct-four/recompute routes verify through the tested per-device context $L=262{,}144$ (Table~\ref{tab:terminal-short}). Within this matched follow-up, direct-four/recompute has the lower VJP median; the result is feasibility evidence for the exact low-state representation, not a latency win.
\begin{table}[htbp]\centering\small
\begin{tabular}{@{}rlrrl@{}}\toprule
$L$ & Method & Complete VJP (ms) & Compiler MiB & Result\\\midrule
16,384 & One-reference & 364.931 & 18.395 & verified\\
 & Direct-four/recompute & 328.927 & 18.332 & verified\\
131,072 & One-reference & 2,992.535 & 354.684 & verified\\
 & Direct-four/recompute & 2,709.000 & 354.622 & verified\\
262,144 & One-reference & 6,020.270 & 835.715 & verified\\
 & Direct-four/recompute & 5,415.049 & 835.528 & verified\\\bottomrule
\end{tabular}
\caption{Digest-normalized windowed TPU v6e follow-up. Each row used one synchronized warmup and three synchronized measurements. Compiler MiB is a backend static executable estimate, not dynamic HBM; 262,144 is a tested point, not a maximum feasible length.}\label{tab:terminal-short}
\end{table}
Compiler totals are static executable estimates with unresolved data-parallel byte scope, not dynamic peak HBM.

The matched 100-update Pfam reconstruction screen has six scientifically verified Sinkhorn/\allowbreak FlashAttention rows for seeds 0, 1, 2: five ordinary zero-exit completions and one content-addressed Sinkhorn seed-2 recovery admitted for scientific metrics only. Sinkhorn reconstruction MSE was 0.0545948/\allowbreak 0.0527878/\allowbreak 0.0523736 and trained FlashAttention was 0.0528423/\allowbreak 0.0506797/\allowbreak 0.0507820; every observed pair favors FlashAttention, with descriptive mean Sinkhorn-minus-FlashAttention difference $+0.0018174$. This is limited-seed, short-horizon reconstruction evidence, not convergence, population, timing, or biological-alignment evidence.

In a separate seed-0 direct-objective calibration, every tested CE weight (0.1, 0.3, 1.0) slightly improved reconstruction MSE, sparse CE, and barycenter MAE at step 100, but each slightly reduced within-five accuracy. No weight passed the prespecified four-part rule, so no checkpoint was selected and seed-1/\allowbreak 2 replication was not triggered.

\section{Discussion and Conclusion}
The contribution is exact fixed-depth backward theory for a balanced fixed-support surrogate, with a single-active-dustbin unit-target bridge, centered-kernel validation, and scoped TPU feasibility and trainability evidence. Exactness is an algebraic property of the stated derivative, not a claim that every floating-point comparison has negligible relative error. The follow-up measurements make the implementation tradeoff explicit: the one-reference representation reduces resident plan state, but the current kernel is not a universal latency winner. Trained FlashAttention is stronger on the short reconstruction screen, and the direct CE calibration does not pass its prespecified alignment selector. This is not a full KL-unbalanced gap model, a global masked-support guarantee, a matched official-FlashSinkhorn backward comparison, a dynamic-HBM measurement, or evidence of biological-alignment superiority.

\paragraph{Funding.} This research received no third-party funding or third-party support. All project expenses were paid personally by the author.
\paragraph{Competing interests.} The author declares no competing interests.
\paragraph{AI assistance.} AI tools assisted with reference checking, exposition, mathematical consistency review, and source reconstruction during this revision. Responsibility for the text, proofs, citations, and empirical claims remains with the author. This editorial revision does not constitute a rerun of the reported training or accelerator experiments.

\clearpage
\bibliographystyle{plainnat}
\bibliography{references}

@article{beltagy2020longformer,
  title={Longformer: The Long-Document Transformer},
  author={Beltagy, Iz and Peters, Matthew E. and Cohan, Arman},
  journal={arXiv preprint arXiv:2004.05150},
  year={2020}
}

@article{cuturi2013sinkhorn,
  title={Sinkhorn Distances: Lightspeed Computation of Optimal Transport},
  author={Cuturi, Marco},
  journal={Advances in Neural Information Processing Systems},
  volume={26},
  year={2013}
}

@inproceedings{luise2018differential,
  title={Differential Properties of {Sinkhorn} Approximation for Learning with {Wasserstein} Distance},
  author={Luise, Giulia and Rudi, Alessandro and Pontil, Massimiliano and Ciliberto, Carlo},
  booktitle={Advances in Neural Information Processing Systems},
  volume={31},
  pages={5864--5874},
  year={2018}
}

@article{peyre2019computational,
  title={Computational Optimal Transport: With Applications to Data Science},
  author={Peyr{\'e}, Gabriel and Cuturi, Marco},
  journal={Foundations and Trends in Machine Learning},
  volume={11},
  number={5--6},
  pages={355--607},
  year={2019},
  doi={10.1561/2200000073}
}

@inproceedings{zaheer2020bigbird,
  title={Big Bird: Transformers for Longer Sequences},
  author={Zaheer, Manzil and Guruganesh, Guru and Dubey, Kumar Avinava and Ainslie, Joshua and Alberti, Chris and Onta{\~n}{o}n, Santiago and Pham, Philip and Ravula, Anirudh and Wang, Qifan and Yang, Li and Ahmed, Amr},
  booktitle={Advances in Neural Information Processing Systems},
  volume={33},
  pages={17283--17297},
  year={2020}
}

@article{pauwels2023derivatives,
  title = {The Derivatives of {Sinkhorn--Knopp} Converge},
  author = {Pauwels, Edouard and Vaiter, Samuel},
  journal = {SIAM Journal on Optimization},
  volume = {33},
  number = {3},
  pages = {1494--1517},
  year = {2023},
  doi = {10.1137/22M1512703},
  url = {https://epubs.siam.org/doi/10.1137/22M1512703}
}

@inproceedings{dao2022flashattention,
  title = {{FlashAttention}: Fast and Memory-Efficient Exact Attention with {IO}-Awareness},
  author = {Dao, Tri and Fu, Daniel Y. and Ermon, Stefano and Rudra, Atri and R{\'e}, Christopher},
  booktitle = {Advances in Neural Information Processing Systems},
  volume = {35},
  year = {2022},
  url = {https://arxiv.org/abs/2205.14135}
}

@misc{ye2026flashsinkhorn,
  howpublished = {arXiv preprint arXiv:2602.03067},
  title = {{FlashSinkhorn}: {IO}-Aware Entropic Optimal Transport on {GPU}},
  author = {Ye, Felix X.-F. and Li, Xingjie and Yu, An and Chang, Ming-Ching and Chu, Linsong and Wertheimer, Davis},
  year = {2026},
  archivePrefix = {arXiv},
  eprint = {2602.03067},
  doi = {10.48550/arXiv.2602.03067},
  url = {https://arxiv.org/abs/2602.03067}
}

@inproceedings{bai2019deep,
 author={Bai, Shaojie and Kolter, J. Zico and Koltun, Vladlen},
 title={Deep Equilibrium Models},
 booktitle={Advances in Neural Information Processing Systems},
 volume={32}, year={2019}
}

@article{birkhoff1957extensions,
 author={Birkhoff, Garrett},
 title={Extensions of {Jentzsch}'s Theorem},
 journal={Transactions of the American Mathematical Society},
 volume={85}, number={1}, pages={219--227}, year={1957},
 doi={10.1090/S0002-9947-1957-0087058-6}
}

@article{bushell1973hilbert,
 author={Bushell, P. J.},
 title={{Hilbert}'s Metric and Positive Contraction Mappings in a {Banach} Space},
 journal={Archive for Rational Mechanics and Analysis},
 volume={52}, number={4}, pages={330--338}, year={1973},
 doi={10.1007/BF00247467}
}

@article{gaubert2015dobrushin,
 author={Gaubert, St{\'e}phane and Qu, Zheng},
 title={{Dobrushin}'s Ergodicity Coefficient for {Markov} Operators on Cones},
 journal={Integral Equations and Operator Theory},
 volume={81}, number={1}, pages={127--150}, year={2015},
 doi={10.1007/s00020-014-2193-2}
}

@article{lin2023evolutionary,
 author={Lin, Zeming and Akin, Halil and Rao, Roshan and Hie, Brian and Zhu, Zhongkai and Lu, Wenting and Smetanin, Nikita and Verkuil, Robert and Kabeli, Ori and Shmueli, Yaniv and others},
 title={Evolutionary-Scale Prediction of Atomic-Level Protein Structure with a Language Model},
 journal={Science}, volume={379}, number={6637}, pages={1123--1130}, year={2023},
 doi={10.1126/science.ade2574}
}

@article{mistry2021pfam,
 author={Mistry, Jaina and Chuguransky, Sara and Williams, Lowri and Qureshi, Matloob and Salazar, Gustavo A. and others},
 title={{Pfam}: The Protein Families Database in 2021},
 journal={Nucleic Acids Research}, volume={49}, number={D1}, pages={D412--D419}, year={2021},
 doi={10.1093/nar/gkaa913}
}

@article{smith1981identification,
 author={Smith, Temple F. and Waterman, Michael S.},
 title={Identification of Common Molecular Subsequences},
 journal={Journal of Molecular Biology}, volume={147}, number={1}, pages={195--197}, year={1981},
 doi={10.1016/0022-2836(81)90087-5}
}
\clearpage
\appendix
\section{Balanced Tail-Refinement Proofs}\label{app:proofs}
\paragraph{Notation.}
For the appendix proofs we work with ungauged representatives satisfying the raw masked logsumexp half-step equations. Let $\Omega$ denote the active support and define
\[
 P_{ij}^{(p,q)}=\begin{cases}\exp(S_{ij}+u_i^{(p)}+v_j^{(q)}),&(i,j)\in\Omega,\\0,&(i,j)\notin\Omega.\end{cases}
\]
\begin{align}
 u_i^{(t+1)}&=-\log\Bigl[\sum_{j:(i,j)\in\Omega}\exp(S_{ij}+v_j^{(t)})\Bigr],\label{eq:row-app}\\
 v_j^{(t+1)}&=-\log\Bigl[\sum_{i:(i,j)\in\Omega}\exp(S_{ij}+u_i^{(t+1)})\Bigr].\label{eq:col-app}
\end{align}
\paragraph{Centering map.}
Let $\one_{\rm row},\one_{\rm col}$ denote the all-ones vectors on valid rows and columns, and let $\omega_{\rm row}=\one_{\rm row}/(\one_{\rm row}^\top\one_{\rm row})$. The implementation stores centered representatives via
\[
 \mathcal C(u,v)=\bigl(u-(\omega_{\rm row}^\top u)\one_{\rm row},\ v+(\omega_{\rm row}^\top u)\one_{\rm col}\bigr).
\]
\begin{lemma}[Centering Pullback]\label{lem:center}
If $(\bar u,\bar v)$ are cotangents on the outputs of $\mathcal C$, then
\[
 \mathcal C^*(\bar u,\bar v)=\bigl(\bar u+\omega_{\rm row}(\one_{\rm col}^\top\bar v-\one_{\rm row}^\top\bar u),\ \bar v\bigr).
\]
\end{lemma}
\begin{proof}
Write $s=\omega_{\rm row}^\top u$, so $\delta s=\omega_{\rm row}^\top\delta u$. Then
\[
 \delta\mathcal C(u,v)=\bigl(\delta u-(\delta s)\one_{\rm row},\ \delta v+(\delta s)\one_{\rm col}\bigr).
\]
Pairing this variation with $(\bar u,\bar v)$ and collecting coefficients of $\delta u,\delta v$ yields the stated pullback.
\end{proof}
\begin{lemma}[Gauge Equivariance]\label{lem:gauge}
Let $\Gamma_c(u,v)=(u-c\one_{\rm row},v+c\one_{\rm col})$. Then the raw two-half-step map defined by equations~\eqref{eq:row-app}--\eqref{eq:col-app} is gauge equivariant.
\end{lemma}
\begin{proof}
Adding $c$ to every valid key potential shifts every logsumexp argument in the $u$-update by the same constant, so the update subtracts $c$ from every valid query potential. The $v$-update is identical with rows and columns swapped.
\end{proof}
\begin{proposition}[Centered Gauge Ledger]\label{prop:gaugeledger}
Let $(\widehat u^{(t)},\widehat v^{(t)})$ be the ungauged iterates and $(u^{(t)},v^{(t)})$ the iterates produced by the centered implementation, both started from the same centered base pair. Set $c_0=0$. Then for every $t\ge0$ there exists a scalar $c_t$ such that
\[
 (u^{(t)},v^{(t)})=\Gamma_{c_t}(\widehat u^{(t)},\widehat v^{(t)}).
\]
Consequently same-time plans are gauge-invariant,
\[
 e^{S_{ij}+u_i^{(t)}+v_j^{(t)}}=e^{S_{ij}+\widehat u_i^{(t)}+\widehat v_j^{(t)}},
\]
but mixed-time staircase plans carry scalar gauge factors:
\[
 P_{{\rm cent}\,ij}^{(p,q)}=e^{c_q-c_p}P_{{\rm ung}\,ij}^{(p,q)},\qquad
 P_{{\rm ung}\,ij}^{(p,q)}=e^{c_p-c_q}P_{{\rm cent}\,ij}^{(p,q)}.
\]
The mixed factors are therefore not gauge-invariant unless $p=q$ or $c_p=c_q$.
\end{proposition}
\begin{proof}
The claim at $t=0$ follows from the common base pair and $c_0=0$; the claim at $t=1$ follows from the definition of the centering step. Inductively, if the two previous representatives differ by a gauge transform, then gauge equivariance of the raw refinement step shows that the next raw outputs also differ by a gauge transform; applying $\mathcal C$ merely chooses another representative of the same orbit. The same-time identity follows because the two gauge shifts cancel inside $u_i^{(t)}+v_j^{(t)}$. For mixed times,
\[
 u_i^{(p)}+v_j^{(q)}=\widehat u_i^{(p)}+\widehat v_j^{(q)}-c_p+c_q,
\]
which gives the scalar factor.
\end{proof}
\begin{lemma}[When Centering Pullbacks Vanish]\label{lem:vanish}
For cotangents $(\bar u,\bar v)$ on the outputs of $\mathcal C$, if $\one_{\rm row}^\top\bar u=\one_{\rm col}^\top\bar v$, then $\mathcal C^*(\bar u,\bar v)=(\bar u,\bar v)$. In the unit-target tail reverse pass, this equal-total condition holds at every full-step centering boundary.
\end{lemma}
\begin{proof}
The first claim is immediate from the displayed formula for $\mathcal C^*$. For the tail reverse pass, the terminal transport application gives
\[
 \one_{\rm row}^\top g_u=\sum_{ij}P_{ij}^{(R,R)}Z_{ij}=\one_{\rm col}^\top g_v,
\]
so the final centering pullback contributes no extra term. After pulling $\bar v^{(t)}$ through a same-time $v$-half-step, the same-time plan has unit column sums, hence
\[
 \one_{\rm row}^\top[-P^{(t,t)}\bar v^{(t)}]=-\one_{\rm col}^\top\bar v^{(t)}.
\]
The resulting $\bar u^{(t)}$ has zero total once the direct terminal contribution has been combined at $t=R$, and thereafter the invariant is zero total on both components. Pulling a zero-total $\bar u^{(t)}$ through a unit-row-normalized $u$-half-step gives a zero-total $\bar v^{(t-1)}$. Thus each later centering boundary sees equal totals, so the explicit $\mathcal C^*$ correction vanishes in the unit-target proof. If non-unit marginals are introduced, the corresponding normalized conditional kernels must be used in this invariant.
\end{proof}
\begin{lemma}[Half-Step Jacobians]\label{lem:half}
For every refinement step $t\ge1$,
\begin{align}
 \frac{\partial u_i^{(t)}}{\partial v_j^{(t-1)}}&=-P_{ij}^{(t,t-1)},\label{eq:jac-u}\\
 \frac{\partial v_j^{(t)}}{\partial u_i^{(t)}}&=-P_{ij}^{(t,t)}.\label{eq:jac-v}
\end{align}
\end{lemma}
\begin{proof}
Differentiate the masked logsumexp updates. For example,
\[
 \frac{\partial u_i^{(t)}}{\partial v_j^{(t-1)}}
 =-\frac{\exp(S_{ij}+v_j^{(t-1)})}{\sum_{k:(i,k)\in\Omega}\exp(S_{ik}+v_k^{(t-1)})}
 =-\exp(S_{ij}+u_i^{(t)}+v_j^{(t-1)})=-P_{ij}^{(t,t-1)}.
\]
The second identity is identical with rows and columns swapped.
\end{proof}

\paragraph{Proof of Proposition~\ref{prop:adjoint}.}
\begin{proof}
By the preceding centering-pullback lemma, the explicit $\mathcal C^*$ correction terms vanish at the full-step centering boundaries for the unit-target tail reverse pass. It therefore suffices to write the reverse equations for the raw half-steps, with the centered gauge ledger applied separately when translating mixed-time factors between representatives. The direct terminal transport application contributes $P^{(R,R)}\odot Z$ to $\bar S$ together with
\[
 \bar v^{(R)}=g_v,\qquad \bar u_{\rm direct}^{(R)}=g_u.
\]
Before moving to the preceding $u$-half-step, this terminal $\bar v^{(R)}$ must be pulled through the final $v^{(R)}$ update. This contributes
\[
 \bar S\mathrel{-}=P^{(R,R)}\odot(\one\bar v^{(R)\top}),\qquad
 \bar u^{(R)}=\bar u_{\rm direct}^{(R)}-P^{(R,R)}\bar v^{(R)}.
\]
Now fix a reverse step $t\in\{1,\ldots,R\}$. Pulling $\bar u^{(t)}$ back through the $u$-half-step uses the Jacobian $-P^{(t,t-1)}$:
\[
 \bar v^{(t-1)}=-(P^{(t,t-1)})^\top\bar u^{(t)},\qquad
 \bar S\mathrel{-}=P^{(t,t-1)}\odot(\bar u^{(t)}\one^\top).
\]
For $t\ge2$, pulling $\bar v^{(t-1)}$ back through the $v$-half-step uses the Jacobian $-P^{(t-1,t-1)}$:
\[
 \bar u^{(t-1)}=-P^{(t-1,t-1)}\bar v^{(t-1)},\qquad
 \bar S\mathrel{-}=P^{(t-1,t-1)}\odot(\one\bar v^{(t-1)\top}).
\]
For $t=1$, the $u$-half-step pullback produces the stopped-base cotangent $\bar v^{(0)}$ and the final mixed-time score term $P^{(1,0)}\odot(\bar u^{(1)}\one^\top)$. No earlier $v$-half-step is pulled back inside the differentiable tail because the base pair is stopped. Summing the direct terminal term with the reverse-time half-step contributions yields equation~\eqref{eq:score}.
\end{proof}
\paragraph{Proof of Proposition~\ref{prop:one-ref}.}
\begin{proof}
For any two staircase plans,
\[
 \frac{P_{ij}^{(p,q)}}{P_{ij}^{(2,2)}}=
 \exp(u_i^{(p)}-u_i^{(2)}+v_j^{(q)}-v_j^{(2)}),
\]
which gives the exact reconstruction formula immediately. For the three non-reference factors in the $R=2$ staircase this gives
\[
 P_{ij}^{(2,1)}=P_{ij}^{(2,2)}\beta_j,\quad
 P_{ij}^{(1,1)}=P_{ij}^{(2,2)}\alpha_i\beta_j,\quad
 P_{ij}^{(1,0)}=P_{ij}^{(2,2)}\alpha_i\delta_j.
\]
Substituting these three identities into the expanded VJP in equation~\eqref{eq:r2-expanded} and factoring out $P_{ij}^{(2,2)}$ gives
\[
 \bar S_{ij}=P_{ij}^{(2,2)}
 [Z_{ij}-\bar v_j^{(2)}-\beta_j\bar u_i^{(2)}-\alpha_i\beta_j\bar v_j^{(1)}-\alpha_i\delta_j\bar u_i^{(1)}].
\]
Thus every entry of the direct four-factor adjoint is reproduced exactly from the reference plan and vector modifiers.

For a tile $I\times J$, compute $P_{I,J}^{(2,2)}$ once from $S_{I,J},u_I^{(2)},v_J^{(2)}$, compute $Z_{I,J}$ from the resident slices of $G_I,V_J$, and form the bracketed modifier using only the $I$-slices of $\alpha,\bar u^{(2)},\alpha\bar u^{(1)}$ and the $J$-slices of $\beta,\delta,\bar v^{(2)},\beta\bar v^{(1)}$. Multiplying the single resident reference tile by this modifier yields the exact score-space VJP tile. The value-gradient term uses the same resident reference plan tile, since $\partial\mathcal L/\partial V$ receives $(P_{I,J}^{(2,2)})^\top G_I$. All subsequent propagation to $Q,K$ is linear in the score cotangent tile, so it can be accumulated before the tile is evicted. The stated storage and arithmetic bounds follow by streaming over the $O(LW/B^2)$ active tiles with fixed $B,d,W$, and fixed tail depth.
\end{proof}

\section{Proofs for the Dustbin Bridge and Surrogate Bias}
\paragraph{Proof of Proposition~\ref{prop:dustbin}.}
\begin{proof}
After constructing the augmented tensors and masks, the implementation calls the same unit-target Sinkhorn forward and backward path as in Section~2.2. Because only the first appended token in each dustbin block is active, every transport entry touching the remaining fillers is identically zero. Because the effective band width is $\widetilde W=\max(W,L)$, every active base token can reach the active dustbin token on the opposite side, and conversely the active dustbin query can reach every active base key. Thus the implemented path is exactly the unit-target fixed-support surrogate on an enlarged state space. The proposition is an algebraic code-path statement; the cost of the current widened-band realization is governed by $\widetilde W$, as stated in Section~\ref{sec:dustbin}.
\end{proof}
\paragraph{Proof of Corollary~\ref{cor:dustbin}.}
\begin{proof}
Apply Propositions~\ref{prop:adjoint} and~\ref{prop:one-ref} to the augmented tensors and support from Proposition~\ref{prop:dustbin}. No new algebra is introduced after the augmentation itself.
\end{proof}
\paragraph{Proof of Theorem~\ref{thm:bias}.}
\begin{proof}
By the chain rule,
\[
 \nabla_x\bigl(\Psi_x^R(z_T(x))\bigr)=
 \left[D_x\Psi_x^R(z)\big|_{z=z_T(x)}\right]^\top+
 D_xz_T(x)^\top\nabla_z\Psi_x^R(z_T(x)),
\]
which gives the exact bias identity after subtracting the stopped-base term. The chain rule through the $R$ refinement steps and the terminal scalar objective gives
$\|\nabla_z\Psi_x^R(z_T(x))\|\le L_\ell\rho^R$.
Differentiate the base recursion $z_{t+1}(x)=\Phi_x(z_t(x))$:
\[
 D_xz_{t+1}(x)=D_x\Phi_x(z_t(x))+D_z\Phi_x(z_t(x))D_xz_t(x).
\]
Because $z_0$ is independent of $x$, repeated application of the norm bounds yields
\[
 \|D_xz_T(x)\|\le L_\Phi\sum_{j=0}^{T-1}\rho^j\le\frac{L_\Phi}{1-\rho}.
\]
Combining the two estimates with the exact bias identity gives
$\|g_{\full}^{T,R}(x)-g_{\surr}^{T,R}(x)\|\le L_\ell L_\Phi\rho^R/(1-\rho)$.
\end{proof}
\paragraph{Proof of Proposition~\ref{prop:hilbert}.}
\begin{proof}
Hilbert's projective metric on the positive cone is
\[
 d_H(x,y)=\log\frac{\max_i x_i/y_i}{\min_i x_i/y_i}.
\]
Positive diagonal scaling and coordinatewise inversion preserve this metric: diagonal scaling cancels in the coordinate ratios, and inversion swaps the maximum and minimum ratios. By the Birkhoff--Hopf theorem, any positive linear map $\mathcal B$ with finite projective diameter $\Delta(\mathcal B)$ satisfies
$d_H(\mathcal Bx,\mathcal By)\le\tanh(\Delta(\mathcal B)/4)d_H(x,y)$.
Decompose the two-half-step map as
\[
 w\mapsto Kw\mapsto a\oslash(Kw)
 \mapsto K^\top(a\oslash(Kw))
 \mapsto b\oslash K^\top(a\oslash(Kw)).
\]
The two linear maps contribute the respective Birkhoff--Hopf factors
\[
 \tanh(\Delta(K)/4)\quad\text{and}\quad\tanh(\Delta(K^\top)/4).
\]
The marginal scalings and inversions are projective isometries. Therefore
\[
 d_H(\mathcal T(w),\mathcal T(w'))\le
 \tanh\!\left(\frac{\Delta(K)}4\right)
 \tanh\!\left(\frac{\Delta(K^\top)}4\right)d_H(w,w').
\]
Strict positivity on the active block gives finite projective diameters, hence $\rho_H<1$. For log scalings, $d_H(e^v,e^{v'})=\osc(v-v')$. Additive centering does not change oscillation. This gives the log-potential contraction. Finally, for the positive Gibbs kernel $K_{ij}=\exp(S_{ij})$,
\begin{align*}
 \Delta(K)&=\log\max_{i,i',j,j'}
 \frac{K_{ij}K_{i'j'}}{K_{ij'}K_{i'j}}\\
 &=\max_{i,i',j,j'}(S_{ij}+S_{i'j'}-S_{ij'}-S_{i'j}).
\end{align*}
If $\osc_\Omega(S)=\max S-\min S$ on the active block, the last expression is at most $2\osc_\Omega(S)$. The same bound holds for $K^\top$. Substituting into the Birkhoff--Hopf coefficient yields $\rho_H\le\tanh^2(\osc_\Omega(S)/2)$, and the scaled-dot-product-logit form follows from $S=A/\eps$.
\end{proof}
\paragraph{Proof of Proposition~\ref{prop:posterior}.}
\begin{proof}
By definition,
\[
 g_{\full}(x)=\nabla_xF_R(x,z_T(x)),\qquad
 g_{\surr}(x)=\left[D_xF_R(x,z)\big|_{z=z_T(x)}\right]^\top.
\]
Applying the chain rule to the first expression gives
\[
 \nabla_xF_R(x,z_T(x))=
 \left[D_xF_R(x,z)\big|_{z=z_T(x)}\right]^\top+
 D_xz_T(x)^\top\nabla_zF_R(x,z)\big|_{z=z_T(x)}.
\]
Subtracting the stopped-base partial derivative gives the identity with $\eta_R=\nabla_zF_R(x,z_T(x))$. The norm bound follows immediately from submultiplicativity for the chosen compatible operator norm.
\end{proof}
\paragraph{Proof of Corollary~\ref{cor:selection}.}
\begin{proof}
Proposition~\ref{prop:posterior} gives the exact identity $\Delta g_R=e_R(x)$. Therefore $\|e_R(x)\|\le\tau_{\rm cert}$ implies $\|\Delta g_R\|\le\tau_{\rm cert}$. If only an operator-norm bound $\kappa_T(x)\ge\|D_xz_T(x)^\top\|$ is used, the bound in Proposition~\ref{prop:posterior} gives
\[
 \|\Delta g_R\|\le\|D_xz_T(x)^\top\|\,\|\eta_R(x)\|
 \le\kappa_T(x)\|\eta_R(x)\|\le\tau_{\rm cert}.
\]
For the optimization claim, let $R^\star$ be the first feasible depth. Every smaller depth is infeasible by definition of $R^\star$, and every larger feasible depth has cost at least $\operatorname{Cost}(R^\star)$ because Cost is nondecreasing. Hence $R^\star$ minimizes Cost over the certified feasible set.
\end{proof}
\paragraph{Proof of Proposition~\ref{prop:orbit}.}
\begin{proof}
All plans in the collection use the same masked score tensor $S$ and the same active support. If an $\eps$-scaling schedule is used, this proof applies separately inside each fixed-temperature stage after replacing $S$ by $S^{(s)}=A/\eps_s$; plans from different stages are not generally in one common row/column scaling orbit. For any two dual pairs on that support,
\[
 \frac{P_{ij}^{(p,q)}}{P_{ij}^{(p_{\rm ref},q_{\rm ref})}}=
 \exp(u_i^{(p)}-u_i^{(p_{\rm ref})}+v_j^{(q)}-v_j^{(q_{\rm ref})}),
\]
which gives the orbit formula entrywise. Multiplying this identity by $H_\ell$ and summing over $\ell$ gives the single-reference representation for $\Xi$. The reverse recurrences in Propositions~\ref{prop:adjoint} and~\ref{prop:posterior} require only products of these plans with row or column cotangent vectors and score-space outer factors, all of which are instances of this representation. Each such product can therefore be evaluated by streaming one reference plan tile and applying the appropriate row and column multipliers. With fixed $T,R,W,d$, each half-step contributes $O(LW)$ arithmetic and $O(L)$ vector state, while the input activations require $O(Ld)$ storage.
\end{proof}

\section{Experimental Protocol Notes}\label{app:protocol}
The current manuscript uses three empirical layers:
\begin{enumerate}[leftmargin=*]
\item exactness checks against exact autodiff on small masked problems;
\item a four-configuration TPU v6e-8 Pfam screen across balanced and dustbin-augmented paths;
\item a promoted long-budget balanced run for the factorization-backed $R=2$ path, plus held-out checkpoint and deterministic diagonal-reference evaluation.
\end{enumerate}
All TPU results use a single Cloud TPU v6e-8 host. The short screen uses 50,000 Pfam training pairs with the reported benchmark stack, and the promoted run uses the same benchmark family at longer budget. The biological inputs and language-model-derived embeddings are pre-existing assets rather than new data collection.

For the promoted balanced run, the objective weights were
\[
 \lambda_{\rm rec}=1,\qquad\lambda_{\rm ce}=0,\qquad\lambda_{\rm mono}=0,
\]
so reconstruction loss was optimized directly while CE and monotonicity were logged diagnostically. Here reconstruction is the mean supervised value-vector MSE: for each active annotated query row $i$ with target column $c_i$, we compare the transport output $\sum_jP_{ij}V_j$ to $V_{c_i}$, sum squared error over embedding dimensions, and average over active rows and heads. The original anonymized supplementary code bundle contains the single-file balanced and dustbin-augmented transport kernel analyzed by the theory, a Sinkhorn-only Pfam training harness, the held-out evaluator, exactness and metric summaries, and a TPU-facing backward benchmark harness. The revised source additionally records the terminal reviewer-motivated result tables and content hashes in \texttt{REBUTTAL\_EVIDENCE\_PROVENANCE.md}. Neither package is a complete cloud reproduction bundle: private Pfam/ESM-2 shards, checkpoints, raw logs, embedding generation, full comparison adapters, measured TPU archives, and launchers are omitted. The deterministic diagonal reference maps each active query ordinal to the corresponding length-scaled active key ordinal and uses the same value projection as the selected checkpoint. The empirical claims are therefore scoped to the disclosed protocols and content-bound result artifacts rather than to a complete public data pipeline.

For the repository version used to prepare the manuscript, the synthetic diagnostics are reproduced with the JAX environment active by running \path{python scripts/validate_tail_refinement.py}, \path{python scripts/validate_orbit_factorization.py}, \path{python scripts/validate_tail_bias_certificate.py}, and \path{python scripts/benchmark_four_plan_adjoint.py}. These commands cover the synthetic exactness, orbit, bias-certificate, and local four-plan diagnostics; they do not reproduce the cloud TPU training run without the omitted external shards and environment setup. The supplementary code also includes \path{benchmark_production_backward_modes.py}, a TPU-facing harness that times the optimized tail-refinement backward and compares against exact-autodiff and explicit four-plan modes where feasible. Its optional score-adjoint mode isolates the block-sparse $R=2$ score-cotangent stage. The post-submission complete-step measurements in Appendix~\ref{app:controls} use a separately frozen response-workspace implementation and evidence ledger; the source bundle records their protocol and hashes but does not include the raw cloud result archive.

The main external assets used in the experiments are public Pfam family data \citep{mistry2021pfam} and ESM-2 embeddings from the public ESM-2 checkpoint \citep{lin2023evolutionary}. Pfam is distributed under CC0 through the Pfam/InterPro documentation, and the \texttt{esm2\_t6\_8M\_UR50D} weights used for the embeddings are distributed under the MIT license.

The present cloud-aligned dustbin path is the single-active-dustbin augmented unit-target surrogate. Genuinely KL-unbalanced transport, larger dustbin capacities, sparse-spoke dustbin supports, and differentiable support geometry remain natural follow-on directions, but they are not needed for the main claim carried by this manuscript.

\paragraph{Synthetic diagnostic visual summary.}
Figure~\ref{fig:diagnostic} visualizes the exactness and bias-certificate summaries reported in Tables~\ref{tab:exactness} and~\ref{tab:biascert}. The left panel is generated from the exactness-summary JSON included in the supplementary bundle. The right panel plots the a posteriori bias-certificate rows from Table~\ref{tab:biascert}; the selector tolerance line is the $10^{-5}$ criterion used to choose $R=2$ in the local calibration.
\begin{figure}[htbp]\centering
\includegraphics[width=\linewidth]{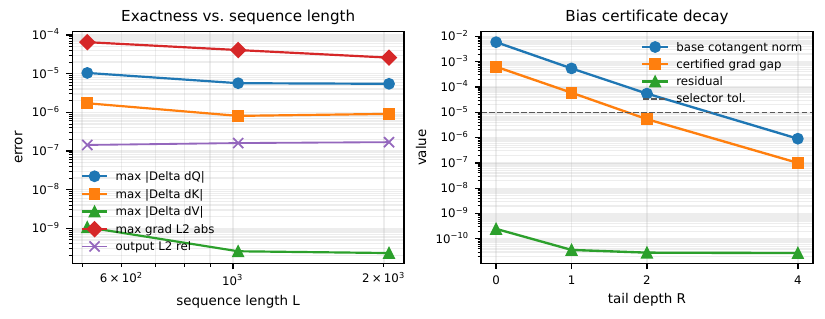}
\caption{Synthetic tail-refinement diagnostics. Left: optimized $R=2$ kernel deviations against exact autodiff through the same centered surrogate. Right: a posteriori bias-certificate decay with tail depth.}\label{fig:diagnostic}
\end{figure}

\paragraph{Qualitative checkpoint output on real Pfam inputs.}
Figure~\ref{fig:pfam-example} visualizes a real Pfam/ESM-2 pair generated by the shard pipeline and evaluated with the saved step-0 and step-1437 Sinkhorn checkpoints. The selected pair has 223 query positions, 212 key positions, and 211 supervised alignment targets. For this example, the promoted checkpoint lowers reconstruction MSE from 4.86 to 1.61 and sparse CE from 5.40 to 5.19. The visualization is qualitative: it shows what the saved final checkpoint outputs on real inputs, while the aggregate held-out metrics remain those in Table~\ref{tab:promoted}.
\begin{figure}[htbp]\centering
\includegraphics[width=\linewidth]{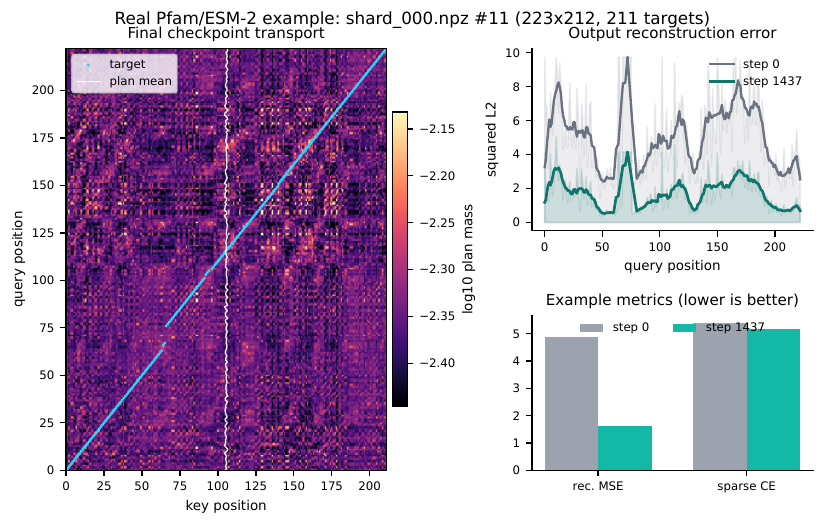}
\caption{Qualitative model output on a real Pfam/ESM-2 pair. Left: final-checkpoint tail-refinement transport plan, with supervised target columns in cyan and the plan barycenter in white. Right: per-position output reconstruction error and example-level lower-is-better metrics before and after promotion.}\label{fig:pfam-example}
\end{figure}

\paragraph{Multiple qualitative checkpoint examples.}
Figure~\ref{fig:pfam-grid} repeats the final-checkpoint transport visualization for four local Pfam/ESM-2 candidates selected by reconstruction-MSE improvement from the first 36 scanned held-out shard examples, subject to at least 180 active targets and maximum visual length 320. The examples have 211--222 supervised targets and improve reconstruction MSE from 4.72--4.86 at step 0 to 1.52--1.61 at step 1437. This figure is a qualitative stress check of the saved checkpoint on multiple real inputs, not an additional benchmark; Table~\ref{tab:promoted} remains the aggregate held-out summary.
\begin{figure}[htbp]\centering
\includegraphics[width=\linewidth]{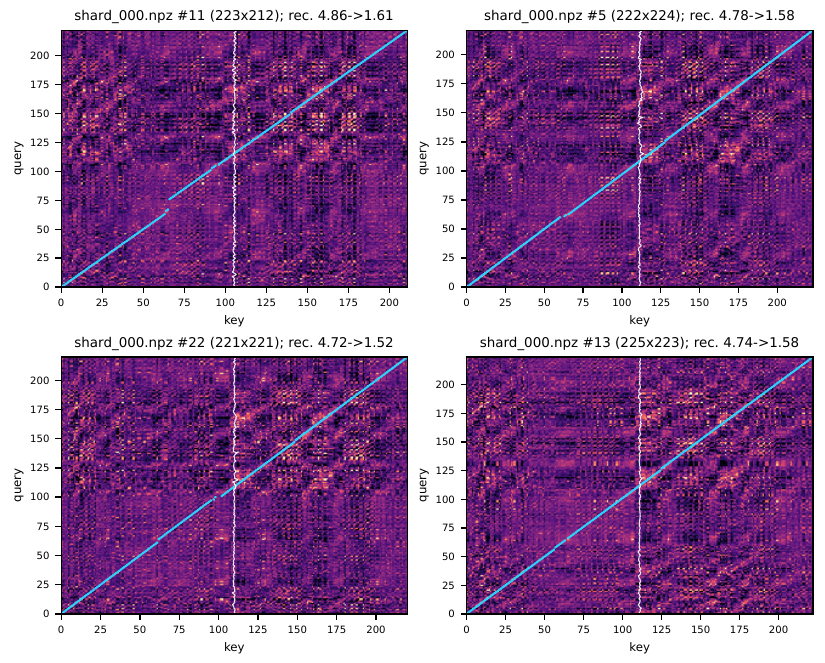}
\caption{Multiple qualitative Pfam/ESM-2 checkpoint examples selected by local reconstruction-MSE improvement. Each panel shows the final-checkpoint tail-refinement transport plan, supervised target columns in cyan, and the plan barycenter in white. Panel titles report sequence lengths and step-0 to step-1437 reconstruction MSE.}\label{fig:pfam-grid}
\end{figure}

\paragraph{Direct four-plan adjoint diagnostic.}
Table~\ref{tab:four-plan} isolates the $R=2$ score-adjoint stage after the base and tail potentials are already available. The direct baseline forms the four staircase factors $(P^{(2,2)},P^{(2,1)},P^{(1,1)},P^{(1,0)})$ separately; the one-reference implementation forms only $P^{(2,2)}$ and applies the row/column modifiers from Proposition~\ref{prop:one-ref}. This is a local dense-JAX CPU diagnostic rather than a TPU systems benchmark, so it should be read as a controlled materialization check, not as a FlashSinkhorn comparison or a production wall-clock claim.
\begin{table}[htbp]\centering\small
\begin{tabular}{@{}rrrrrr@{}}\toprule
$L$ & $W$ & Direct ms & One-ref ms & Speedup & Max $|\Delta\bar S|$\\\midrule
512 & 256 & 2.32 & 1.06 & $2.20\times$ & $4.77\times10^{-7}$\\
1024 & 256 & 7.01 & 5.44 & $1.29\times$ & $1.19\times10^{-7}$\\
2048 & 256 & 60.51 & 39.58 & $1.53\times$ & $2.38\times10^{-7}$\\\bottomrule
\end{tabular}
\caption{Controlled score-adjoint benchmark for direct four-plan materialization versus the one-reference representation. Logical active plan-factor storage is reduced by exactly $4\times$ in all rows: 3.15/\allowbreak 0.79 MB, 7.35/\allowbreak 1.84 MB, and 15.76/\allowbreak 3.94 MB for direct/one-reference respectively. Timings are mean post-compilation CPU times over 20 repetitions.}\label{tab:four-plan}
\end{table}

\paragraph{Analytic memory ledger.}
The production benchmark request ultimately needs measured TPU HBM and wall-clock numbers, which Table~\ref{tab:four-plan} does not provide. To make the memory claim concrete, Table~\ref{tab:memory} gives an analytic ledger for the balanced $R=2$ path at the reported long-context setting $L=16384$, half-band width $W=1024$, block size $B=128$, head dimension $d=64$, and float32 storage. The ledger compares storage inside the $R=2$ adjoint only; it is not a measured comparison against FlashAttention or FlashSinkhorn workspaces. The command \path{python scripts/estimate_r2_backward_memory.py} reproduces these values.
\begin{table}[htbp]\centering\small
\begin{tabular}{@{}lrr@{}}\toprule
Quantity & Direct four-plan & One-reference\\\midrule
Active support entries per head & \multicolumn{2}{c}{32.52M}\\
Logical materialized plan factors & 496.23 MiB & 124.06 MiB\\
Resident plan tile at $B=128$ & 0.2500 MiB & 0.0625 MiB\\
$R=2$ tail $u/v$ vectors & \multicolumn{2}{c}{0.375 MiB}\\
Base-plus-tail $u/v$ history upper ledger & \multicolumn{2}{c}{2.25 MiB}\\
QKV activations & \multicolumn{2}{c}{12.00 MiB}\\\bottomrule
\end{tabular}
\caption{Analytic storage ledger for the balanced banded $R=2$ path at $L=16384$, $W=1024$, $B=128$, $d=64$ in float32. The first two rows quantify logical plan-factor materialization; the resident-tile row reflects the streaming block schedule. This is not a measured TPU HBM profile.}\label{tab:memory}
\end{table}

\paragraph{Projective contraction diagnostics on Pfam checkpoints.}
Table~\ref{tab:projective} reports the Birkhoff--Hopf certificate from Proposition~\ref{prop:hilbert} on the local Pfam test shards for the step-0 and step-1437 Sinkhorn checkpoints. For each strictly positive active 128-row block we compute the exact projective-diameter coefficient $\rho_H=\tanh^2(\Delta/4)$ on the scaled scores at $\eps=1$, excluding masked rows and columns. We also report the coarser range-only sufficient bound $\rho_{\rm range}=\tanh^2(\Omega/2)$. Across 228 active blocks from 100 held-out examples, both certificates remain nontrivial; the final checkpoint has median $\rho_H=0.241$ and maximum $\rho_H=0.590$. The increase from step 0 to step 1437 indicates that the trained projections operate at higher score contrast, while remaining inside the locally contractive positive-block regime measured by the certificate.
\begin{table}[htbp]\centering\small
\begin{tabular}{@{}rrrrrrr@{}}\toprule
Step & Blocks & Median $\rho_H$ & P95 $\rho_H$ & Max $\rho_H$ & Median $\rho_{\rm range}$ & Max $\rho_{\rm range}$\\\midrule
0 & 228 & 0.0748 & 0.296 & 0.345 & 0.282 & 0.582\\
1437 & 228 & 0.241 & 0.532 & 0.590 & 0.482 & 0.721\\\bottomrule
\end{tabular}
\caption{Production-checkpoint projective contraction diagnostics on local Pfam test shards. The exact projective certificate is computed from the active block score matrix; the range-only bound is the simpler sufficient bound stated in Proposition~\ref{prop:hilbert}.}\label{tab:projective}
\end{table}

\paragraph{Compact Pfam diagnostic visual summary.}
Figure~\ref{fig:pfam-summary} visualizes two local empirical summaries requested by the review. The left panel converts the step-50 screen table into a column-normalized heatmap, so the short-screen comparison across loss, sparse CE, monotonicity, and throughput is visible at a glance. The right panel shows the full $\rho_H$ distribution underlying Table~\ref{tab:projective}; it indicates higher learned score contrast, hence weaker but still bounded local contraction margins, rather than improved contraction.
\begin{figure}[htbp]\centering
\includegraphics[width=\linewidth]{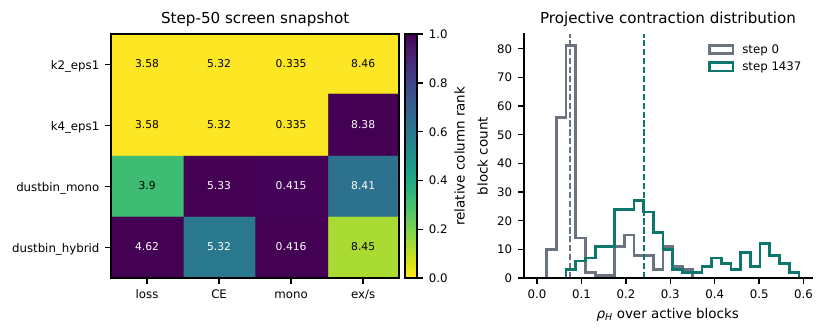}
\caption{Compact Pfam diagnostic summaries. Left: relative column ranks for the short TPU screen metrics, with annotated raw values. Right: exact projective contraction coefficient distributions over active blocks at the step-0 and step-1437 checkpoints.}\label{fig:pfam-summary}
\end{figure}

\section{Kernel-Oriented Pseudocode}
\paragraph{Balanced stopped-base and $R=2$ tail.}
\begin{enumerate}[leftmargin=*]
\item Form masked scaled-dot-product logits $A_{ij}=Q_iK_j^\top/\sqrt d$ and effective scores $S_{ij}=A_{ij}/\eps$ on the active support $\Omega$.
\item Run the balanced masked Sinkhorn half-steps for $T$ stopped iterations to obtain the base pair $(u^{(0)},v^{(0)})$.
\item Recompute two differentiable refinement steps from the stopped base pair, producing $(u^{(1)},v^{(1)})$ and $(u^{(2)},v^{(2)})$.
\item Materialize only the base staircase plan $P^{(2,2)}$ blockwise.
\item Reconstruct $P^{(2,1)},P^{(1,1)},P^{(1,0)}$ by row and column rescaling of $P^{(2,2)}$ using the corresponding dual differences $(u^{(p)}-u^{(2)},v^{(q)}-v^{(2)})$. When using stored centered potentials, include the scalar gauge factors from the centered gauge ledger or reconstruct in a gauge-consistent raw representative.
\item Evaluate the exact $R=2$ score-space VJP from Propositions~\ref{prop:adjoint} and~\ref{prop:one-ref}, then propagate to $(Q,K,V)$.
\end{enumerate}
\paragraph{Dustbin-augmented bridge.}
\begin{enumerate}[leftmargin=*]
\item Append one active dustbin token per side together with inactive fillers to construct the augmented tensors $(\widetilde Q,\widetilde K,\widetilde V)$.
\item Augment the masks analogously and enlarge the effective band width to $\widetilde W=\max(W,L)$.
\item Apply the same balanced stopped-base and $R=2$ tail procedure on the augmented state space.
\end{enumerate}

\section{Post-Submission Experimental Controls}\label{app:controls}
This appendix records experiments run in response to review. They do not change the submitted surrogate or theorem. The evidence was frozen in the separate NeurIPS response workspace and is content-bound in \texttt{REBUTTAL\_EVIDENCE\_PROVENANCE.md}.

All timings below are synchronized medians under their stated local protocols; we do not combine them across hardware, allocations, execution paths, or timing scopes. ``Compiler total'' is the backend static estimate (arguments plus outputs plus temporaries minus aliases), not a measured dynamic-HBM high-watermark. The exact differentiated object is the stopped $T=15$ base plus $R=2$ tail with arbitrary output cotangent, not full BPTT through the stopped base solve.

\subsection{Literal and recomputed four-plan comparisons}
The literal V5E1 comparison uses float32 inputs, $B=H=1$, $d=64$, $W=1024$, one fresh process per mode and length, five warmups, and twenty synchronized measurements. It separates saved-state backward-only timing from a compiled forward/materialization/retention/backward executable. Backward-only medians at $L=2{,}048/\allowbreak 8{,}192/\allowbreak 16{,}384$ are 0.812/\allowbreak 3.004/\allowbreak 6.101 ms for literal four-resident, 2.561/\allowbreak 10.994/\allowbreak 21.743 ms for direct-four/recompute, and 9.722/\allowbreak 47.231/\allowbreak 96.479 ms for one-reference. Exact autodiff has no matched public saved-state consumer and is excluded from that timing scope.
\begin{table}[htbp]\centering\small
\begin{tabular}{@{}rlrrr@{}}\toprule
$L$ & Mode & Complete ms & Compiler total (MiB) & Saved payload (MiB)\\\midrule
2,048 & Four-resident & 30.62 & 70.3 & 64.0\\
 & Direct-four/recompute & 29.70 & 7.0 & 0.0156\\
 & One-reference & 36.90 & 7.0 & 0.0156\\
 & Exact autodiff & 37.25 & 248.3 & 0\\
8,192 & Four-resident & 125.88 & 290.4 & 272.0\\
 & Direct-four/recompute & 128.04 & 18.6 & 0.0625\\
 & One-reference & 164.12 & 18.6 & 0.0625\\
 & Exact autodiff & 188.32 & 1,428.9 & 0\\
16,384 & Four-resident & 257.44 & 578.5 & 544.0\\
 & Direct-four/recompute & 260.11 & 34.7 & 0.125\\
 & One-reference & 334.91 & 34.7 & 0.125\\
 & Exact autodiff & 393.22 & 3,204.1 & 0\\\bottomrule
\end{tabular}
\caption{Literal same-V5E1 complete-step audit. Each row is one compiled forward/materialization/retention/arbitrary-cotangent Q/K/V-backward executable. Saved payload is included in compiler total. Dynamic XProf peak HBM was unavailable.}\label{tab:v5}
\end{table}
At $L=16{,}384$, the one-reference route therefore retains 0.125 MiB of explicit vector state with a 34.7 MiB compiler total, versus 544 MiB of retained plan state and 578.5 MiB total for literal four-resident. Its complete-step median is nevertheless larger. The result is a representation and resident-state trade-off, not a latency win.

The all-eight-device windowed follow-up used BF16 inputs, $H=1$, $d=64$, $W=1024$, $T=15$, $R=2$, seed 0, one process, true data-parallel \texttt{pmap}, aggregate batch eight, one distinct full-length example per device, and no sequence sharding. Every row used one synchronized warmup and three synchronized measurements. Table~\ref{tab:terminal-full} includes forward and complete arbitrary-cotangent Q/K/V VJP medians.
\begin{table}[htbp]\centering\footnotesize
\begin{tabular}{@{}rlrrr@{}}\toprule
Per-device $L$ & Method & Forward ms & Complete VJP ms & Compiler MiB\\\midrule
16,384 & One-reference & 310.817 & 364.931 & 18.395\\
 & Direct-four/recompute & 310.806 & 328.927 & 18.332\\
 & Literal four-resident & 310.718 & 330.386 & 562.301\\
131,072 & One-reference & 2,575.256 & 2,992.535 & 354.684\\
 & Direct-four/recompute & 2,575.260 & 2,709.000 & 354.622\\
262,144 & One-reference & 5,196.908 & 6,020.270 & 835.715\\
 & Direct-four/recompute & 5,197.045 & 5,415.049 & 835.528\\\bottomrule
\end{tabular}
\caption{All-eight-device, full-context-per-device terminal follow-up. Forward and complete-VJP times are medians. Compiler total has unresolved \texttt{pmap} byte scope and is not assumed per-device or aggregate. Dynamic peak HBM is unavailable; 262,144 is a tested point, not a maximum feasible length.}\label{tab:terminal-full}
\end{table}
At 16K, output/dQ/dK/dV are byte-identical to frozen witnesses for one-reference, direct-four/recompute, and literal four-resident. Standalone forward and the primal returned by the complete VJP are byte-identical in every follow-up row. The explicit complete-step route avoids an outer autodiff transformation that triggered an earlier scoped-VMEM lowering boundary; it does not alter the surrogate or its gradients.

The historical Retry12 rows used a separate allocation, two warmups, five measurements, and a different timeout, so their exact-input juxtaposition is descriptive only. On the historical one-device and all-eight-device lowering surfaces, one-reference and direct-four/recompute completed through $L=65{,}536$ and hit the scoped-VMEM compilation boundary at $L=131{,}072$. The terminal explicit route does not retroactively convert those failed rows into successes, and no cross-allocation or cross-topology ratio is formed.

\subsection{Attention and IO-aware diagnostics}
Dense FlashAttention and local SplashAttention complete Q/K/V VJPs through tested $L=262{,}144$ on the historical all-eight-device surface, with medians 7,807.287 and 64.092 ms at that length. Dense FlashAttention uses global support and local SplashAttention is a same-support softmax operator, so neither is an equal-operator Sinkhorn timing comparison.

The official FlashSinkhorn interface (version 0.3.3.post1) was tested on its supported same-A100 dense-forward surface using byte-identical float32 Q/K/V. Official native medians at $L=512/\allowbreak 1{,}024/\allowbreak 2{,}048/\allowbreak 4{,}096$ were 3.717/\allowbreak 3.773/\allowbreak 5.727/\allowbreak 15.549 ms, versus 34.014/\allowbreak 114.653/\allowbreak 411.212/\allowbreak 1,700.474 ms for the forced-dense non-native Pallas path; official execution remained finite through tested $L=50{,}000$. Since $W=L$ and the official interface does not expose the submitted hard-band arbitrary-cotangent Q/K/V backward or matched Pfam training surface, this is a backend diagnostic, not a method-level speed, comparative-OOM, backward, or matched-Pfam result.
\begin{table}[htbp]\centering\small
\begin{tabular}{@{}rrr@{}}\toprule
$L$ & Official native (ms) & Forced-dense Pallas (ms)\\\midrule
512 & 3.717 & 34.014\\
1,024 & 3.773 & 114.653\\
2,048 & 5.727 & 411.212\\
4,096 & 15.549 & 1,700.474\\\bottomrule
\end{tabular}
\caption{Same-A100, byte-identical-Q/K/V dense-forward diagnostic. Native target, arithmetic policy, normalization timing, and timing scope are not equivalent; no method-level ratio is inferred.}\label{tab:official}
\end{table}

\subsection{Trained reconstruction and direct-alignment controls}
\begin{table}[htbp]\centering\small
\begin{tabular}{@{}rrrr@{}}\toprule
Seed & Sinkhorn MSE & Trained FlashAttention MSE & Difference\\\midrule
0 & 0.0545948 & 0.0528423 & $+0.0017525$\\
1 & 0.0527878 & 0.0506797 & $+0.0021081$\\
2 & 0.0523736 & 0.0507820 & $+0.0015915$\\\bottomrule
\end{tabular}
\caption{Matched 100-update, 5,000-example Pfam reconstruction screen. The difference is Sinkhorn minus trained FlashAttention. Recovered Sinkhorn seed 2 contributes scientific evidence only and is excluded from timing, throughput, systems, lifecycle, and ordinary-success claims.}\label{tab:reconstruction}
\end{table}
\begin{table}[htbp]\centering\small
\begin{tabular}{@{}rrrrrl@{}}\toprule
$\lambda_{\rm CE}$ & Reconstruction & Sparse CE & Bary. MAE & Within-five & Pass\\\midrule
0.0 & 0.0545948 & 5.56091 & 77.3981 & 0.0435078 & baseline\\
0.1 & 0.0545891 & 5.55988 & 77.3928 & 0.0434887 & no\\
0.3 & 0.0545882 & 5.55984 & 77.3925 & 0.0434868 & no\\
1.0 & 0.0545910 & 5.55997 & 77.3935 & 0.0434926 & no\\\bottomrule
\end{tabular}
\caption{Seed-0 direct-objective calibration at the prespecified step-100 selection point. A row had to improve sparse CE, barycenter MAE, and within-five while keeping reconstruction degradation at most 10\%. No row passed; intermediate checkpoint crossings are post-hoc diagnostics and were not used for selection.}\label{tab:ce}
\end{table}
All reconstruction rows used the single-device V5E1 hardware class, but physical TPU identity was not held fixed and whole-invocation timing was not a wall-clock-matched kernel comparison. Three seeds provide descriptive limited-seed uncertainty only. Five rows are ordinary zero-exit completions; Sinkhorn seed 2 is a content-addressed scientific recovery and is ineligible for timing, lifecycle, throughput, or ordinary-success claims. The FlashAttention evaluator is the canonical dense float32 reference applied to trained projections, not proof of deployed BF16 fused-kernel parity. Because the direct-objective selector failed, there is no selected-weight replication or multi-seed direct-alignment claim. Neither experiment establishes biological superiority or reverses the strong diagonal-reference alignment result in Table~\ref{tab:promoted}.
\clearpage

\section{Extended Finite-Orbit and Quotient Theory}\label{app:extended}
This arXiv v3 records theory developed after the reviewed manuscript. The $R=2$ one-reference schedule remains the implementation result in the main paper. Proposition~\ref{prop:orbit} already gives the arbitrary finite same-score transport-orbit identity, and Proposition~\ref{prop:hilbert} already gives the strictly positive-block Hilbert certificate; neither result is restated here. The new material below extends the finite-tail adjoint to compatible general marginals, gives concrete counterexamples at the temperature and gauge boundaries, and supplies sparse-support Dobrushin certificates for quotient cotangents.

\subsection{Compatible General Marginals}
Let $\Omega\subseteq[m]\times[n]$ be a fixed active support and let $S$ be a fixed effective score tensor on $\Omega$. Inactive entries are outside the computation graph. Let $a\in\mathbb R_{>0}^m$ and $b\in\mathbb R_{>0}^n$ be compatible positive marginals, $\sum_i a_i=\sum_j b_j$. The log-domain half-steps are
\begin{align*}
 u_i^{(t+1)}&=\log a_i-\log\sum_{j:(i,j)\in\Omega}\exp(S_{ij}+v_j^{(t)}),\\
 v_j^{(t+1)}&=\log b_j-\log\sum_{i:(i,j)\in\Omega}\exp(S_{ij}+u_i^{(t+1)}).
\end{align*}
For row time $p$ and column time $q$, define
\[
 P_{ij}^{(p,q)}=\exp(S_{ij}+u_i^{(p)}+v_j^{(q)})\one_{(i,j)\in\Omega}.
\]
All statements in this subsection concern a finite computation graph with fixed support and fixed effective score. Temperature continuation is handled stagewise.
\begin{lemma}[General-marginal half-step Jacobians]\label{lem:general-half}
The local derivatives are
\[
 \frac{\partial u_i^{(t)}}{\partial v_j^{(t-1)}}=-\frac{P_{ij}^{(t,t-1)}}{a_i},\qquad
 \frac{\partial v_j^{(t)}}{\partial u_i^{(t)}}=-\frac{P_{ij}^{(t,t)}}{b_j}.
\]
The score-entry derivatives through the corresponding half-steps have the same factors for $(i,j)\in\Omega$, and inactive score derivatives are zero.
\end{lemma}
\begin{proof}
For the row half-step, the normalizer is
$\sum_j\exp(S_{ij}+v_j^{(t-1)})=a_i\exp(-u_i^{(t)})$,
so the masked softmax weight is $P_{ij}^{(t,t-1)}/a_i$. The minus sign comes from subtracting the log normalizer. The column half-step is identical with normalizer $b_j\exp(-v_j^{(t)})$. Each active score entry enters the same logsumexp as its corresponding input potential.
\end{proof}
\begin{proposition}[General-marginal finite-tail score VJP]\label{prop:general-vjp}
Let a scalar loss $\mathcal L$ depend on the terminal plan through the edgewise cotangent $Z:=\partial\mathcal L/\partial P^{(R,R)}$. Initialize $\bar S=0$; below, $+=$ and $-=$ denote accumulation into this score cotangent. Define
\[
 \bar v^{(R)}=(P^{(R,R)}\odot Z)^\top\one,\qquad g_u=(P^{(R,R)}\odot Z)\one.
\]
The terminal transport application and the final $v^{(R)}$ half-step give
\begin{align*}
 \bar S&\mathrel{+}=P^{(R,R)}\odot Z-P^{(R,R)}\odot[\one(\bar v^{(R)}\oslash b)^\top],\\
 \bar u^{(R)}&=g_u-P^{(R,R)}(\bar v^{(R)}\oslash b).
\end{align*}
For $t=R,R-1,\ldots,1$, the pullback through the $u^{(t)}$ half-step is
\begin{align*}
 \bar S&\mathrel{-}=P^{(t,t-1)}\odot[(\bar u^{(t)}\oslash a)\one^\top],\\
 \bar v^{(t-1)}&=-(P^{(t,t-1)})^\top(\bar u^{(t)}\oslash a).
\end{align*}
For $t\ge2$, the pullback through the preceding $v^{(t-1)}$ half-step is
\begin{align*}
 \bar S&\mathrel{-}=P^{(t-1,t-1)}\odot[\one(\bar v^{(t-1)}\oslash b)^\top],\\
 \bar u^{(t-1)}&=-P^{(t-1,t-1)}(\bar v^{(t-1)}\oslash b).
\end{align*}
The stopped-tail surrogate discards the resulting base cotangent; the a posteriori certificate may instead pull it through the finite base solve. For $a=b=\one$ this reduces to the adjoint in the main paper.
\end{proposition}
\begin{proof}
The terminal variation is
\[
 \delta\langle P^{(R,R)},Z\rangle=
 \langle P^{(R,R)}\odot Z,\delta S\rangle+
 \langle g_u,\delta u^{(R)}\rangle+\langle\bar v^{(R)},\delta v^{(R)}\rangle.
\]
Before recursing, $\bar v^{(R)}$ must be pulled through the final $v^{(R)}$ update. Lemma~\ref{lem:general-half} yields the same-time score decrement and changes the row-potential cotangent to $g_u-P^{(R,R)}(\bar v^{(R)}\oslash b)$. Applying the two half-step Jacobians in reverse time gives the remaining recurrences. The recursion stops at the stopped base state.
\end{proof}

\subsection{Two Boundaries of Orbit Reconstruction}
The same-score requirement in Proposition~\ref{prop:orbit} is essential. The following witness makes the stagewise limitation concrete.
\begin{proposition}[Cross-temperature obstruction]\label{prop:temperature}
Plans generated from $A/\eps_1$ and $A/\eps_2$ with $\eps_1\ne\eps_2$ need not lie in one row/column scaling orbit. On the complete $2\times2$ support with
\[
 A=\begin{pmatrix}0&0\\0&1\end{pmatrix},
\]
the cross-temperature log-ratio has a nonzero alternating cycle sum.
\end{proposition}
\begin{proof}
Two positive plans on the same support can share a row/column scaling orbit only if their entrywise log-ratio is a row term plus a column term. Across temperatures, that log-ratio contains $(1/\eps_1-1/\eps_2)A_{ij}$ in addition to dual-potential differences. Its alternating sum on the displayed support is
\[
 (1/\eps_1-1/\eps_2)(A_{11}-A_{12}-A_{21}+A_{22})=1/\eps_1-1/\eps_2\ne0.
\]
Every row-plus-column separable array has zero alternating sum, proving the claim.
\end{proof}
The centered gauge ledger in Appendix~\ref{app:proofs} already proves the required scalar correction for mixed-time plans. The next example shows that omitting it changes the score VJP rather than merely changing notation.
\begin{proposition}[A centered-gauge counterexample]\label{prop:gauge-counter}
Let
\[
 S=\begin{pmatrix}0&0.2\\-0.1&0.3\end{pmatrix},\quad
 u^{(p)}=\begin{pmatrix}0.4\\-0.2\end{pmatrix},\quad
 v^{(p)}=\begin{pmatrix}0.1\\-0.3\end{pmatrix},\quad
 v^{(q)}=\begin{pmatrix}-0.2\\0.5\end{pmatrix}.
\]
Choose $c_p=1$ and $c_q=-1/\allowbreak 2$. The same-time plan is unchanged by centering with $c_p$, whereas the centered mixed-time plan is multiplied by $\exp(c_q-c_p)=\exp(-3/\allowbreak 2)$. Multiplication by $\exp(c_p-c_q)$ restores the ungauged plan. Consequently, any nonzero mixed-time VJP term using the uncorrected centered factor is wrong.
\end{proposition}
\begin{proof}
The same-time centered exponent is $S_{ij}+u_i^{(p)}-c_p+v_j^{(p)}+c_p$, so its scalar shifts cancel. The mixed-time exponent is $S_{ij}+u_i^{(p)}-c_p+v_j^{(q)}+c_q$, which differs from the ungauged exponent by $c_q-c_p=-3/\allowbreak 2$. The nonunit factor changes every nonzero edge contribution of the mixed-time VJP term.
\end{proof}
This orbit result concerns resident-plan storage, not low rank of the final score cotangent: the accumulated edgewise modifier can be arbitrary.

\subsection{Sparse-Support Dobrushin Certificates}
Masked zeros prevent the positive-map argument in Proposition~\ref{prop:hilbert} from becoming a global sparse support theorem. A complementary certificate can be built from the concrete full-step derivative. For a vector $h$, define $\osc(h)=\max_i h_i-\min_i h_i$. For a signed cotangent $\eta$ with $\one^\top\eta=0$, define
\[
 \|\eta\|_{\TV}=\frac12\sum_i|\eta_i|.
\]
The zero-mass restriction is dual to quotienting log potentials by additive constants. For a raw column-potential cotangent $\zeta\in\mathbb R^n$, let
\[
 \Pi_0=I_n-\frac1n\one\one^\top,\qquad \Pi_0\zeta=\zeta-\frac{\one^\top\zeta}{n}\one.
\]
For one compatible-marginal full step, define
\[
 R_t=\diag(a)^{-1}P^{(t,t-1)},\qquad
 C_t=P^{(t,t)}\diag(b)^{-1},\qquad M_t=C_t^\top R_t.
\]
$R_t$ is row-stochastic, $C_t$ is column-stochastic, and $M_t$ is row-stochastic.
\begin{lemma}[Full-step derivative orientation]\label{lem:orientation}
For column-vector perturbations of the previous column potential,
\[
 \delta v^{(t)}=M_t\delta v^{(t-1)},\qquad M_t=C_t^\top R_t.
\]
The reverse cotangent map is $M_t^\top$.
\end{lemma}
\begin{proof}
The first half-step gives $\delta u_i^{(t)}=-\sum_k(R_t)_{ik}\delta v_k^{(t-1)}$ and the second gives
\[
 \delta v_j^{(t)}=-\sum_i(C_t)_{ij}\delta u_i^{(t)}
 =\sum_{i,k}(C_t)_{ij}(R_t)_{ik}\delta v_k^{(t-1)}.
\]
This is the displayed orientation. Stochasticity follows from the normalized rows of $R_t$ and normalized columns of $C_t$.
\end{proof}
For a row-stochastic matrix $M$, define the Dobrushin coefficient \citep{gaubert2015dobrushin}
\[
 \tau(M)=1-\min_{p,q}\sum_k\min(M_{pk},M_{qk})
 =\frac12\max_{p,q}\sum_k|M_{pk}-M_{qk}|.
\]
\begin{theorem}[Primal oscillation contraction]\label{thm:osc}
For row-stochastic $M$, $\osc(Mh)\le\tau(M)\osc(h)$.
\end{theorem}
\begin{proof}
Shift $h$ so that its range lies in $[0,\osc(h)]$. The difference between any two output coordinates is the pairing of $h$ with two probability rows. Their difference has zero total mass, so its positive and negative parts have equal mass. Thus
\[
 |(Mh)_p-(Mh)_q|\le\frac12\sum_k|M_{pk}-M_{qk}|\osc(h).
\]
Maximizing over $p,q$ proves the result.
\end{proof}
\begin{theorem}[Reverse zero-mass TV contraction]\label{thm:tv}
For row-stochastic $M$ and $\one^\top\eta=0$,
$\|M^\top\eta\|_{\TV}\le\tau(M)\|\eta\|_{\TV}$.
\end{theorem}
\begin{proof}
On zero-mass cotangents, $\|\eta\|_{\TV}=\sup_{\osc(h)\le1}\langle\eta,h\rangle$. Therefore
\[
 \|M^\top\eta\|_{\TV}=\sup_{\osc(h)\le1}\langle\eta,Mh\rangle
 \le\tau(M)\|\eta\|_{\TV},
\]
using Theorem~\ref{thm:osc}.
\end{proof}
\begin{lemma}[Projection to quotient cotangents]\label{lem:projection}
Let $M$ be row-stochastic and $\bar\eta=\Pi_0\eta$. Then $\one^\top\bar\eta=0$ and $\one^\top M^\top\bar\eta=0$. In general, $\Pi_0M^\top\eta\ne M^\top\Pi_0\eta$; projection is therefore part of the quotient recurrence rather than post-processing of a raw recurrence.
\end{lemma}
\begin{proof}
The first claim is immediate. Since $M\one=\one$, $\one^\top M^\top\bar\eta=(M\one)^\top\bar\eta=0$. Row-stochastic matrices need not be column-stochastic, giving the stated noncommutation in general.
\end{proof}
\begin{theorem}[Homogeneous quotient-cotangent decay]\label{thm:homogeneous}
Let $M_1,\ldots,M_R$ be row-stochastic full-step derivative kernels. Set $\bar\eta_R=\Pi_0\eta_R^{\rm raw}$ unless the terminal cotangent is already zero-mass, and propagate $\bar\eta_{r-1}=M_r^\top\bar\eta_r$. Then
\[
 \|\bar\eta_0\|_{\TV}\le\left(\prod_{r=1}^R\tau(M_r)\right)\|\bar\eta_R\|_{\TV}.
\]
\end{theorem}
\begin{proof}
Lemma~\ref{lem:projection} keeps the recurrence in the zero-mass quotient space. Apply Theorem~\ref{thm:tv} at each step and multiply the inequalities.
\end{proof}
\begin{figure}[htbp]\centering
\begin{tikzpicture}[font=\small,>=Latex]
\node[draw,rounded corners,align=center,minimum width=11cm,inner sep=7pt] (primal) at (0,2.1)
 {Primal constants: $M\one=\one$.\\Cotangent total mass: $\one^\top M^\top\eta=\one^\top\eta$.\\A constant cotangent vector need not be fixed by $M^\top$.};
\node[draw,align=center,inner sep=6pt] (a) at (-4,0.4) {$\bar\eta_R$\\$\one^\top\bar\eta_R=0$};
\node[draw,align=center,inner sep=6pt] (b) at (-1.3,0.4) {$\bar\eta_{R-1}$};
\node (c) at (1.3,0.4) {$\cdots$};
\node[draw,align=center,inner sep=6pt] (d) at (4,0.4) {$\bar\eta_0$\\$\one^\top\bar\eta_0=0$};
\draw[->] (a)--node[above]{$M_R^\top$}(b);
\draw[->] (b)--node[above]{$M_{R-1}^\top$}(c);
\draw[->] (c)--node[above]{$M_1^\top$}(d);
\node[align=center] at (0,-0.9) {$\displaystyle\|\bar\eta_0\|_{\TV}\le
 \left(\prod_{r=1}^R\tau(M_r)\right)\|\bar\eta_R\|_{\TV}$\\[4pt]
Project terminal/source cotangents before propagating the quotient recurrence.};
\end{tikzpicture}
\caption{Dobrushin total-variation contraction of quotient cotangents. The zero-mass space is preserved and contracts under transpose stochastic derivative kernels. Primal constants are fixed by $M$, whereas $M^\top$ conserves cotangent total mass; it does not generally preserve a constant cotangent vector. Raw nonquotient contributions require separate accounting.}\label{fig:quotient}
\end{figure}
Actual finite-tail objectives can inject source cotangents between boundaries. Projection below means projecting terminal and boundary-source cotangents before the recurrence, not propagating raw cotangents and projecting at the end.
\begin{theorem}[Inhomogeneous quotient-cotangent decay]\label{thm:inhomogeneous}
Let $M_1,\ldots,M_R$ be row-stochastic. Define $\bar\eta_R=\Pi_0\eta_R^{\rm raw}$ and $\bar\xi_{s-1}=\Pi_0\xi_{s-1}^{\rm raw}$ unless these cotangents are already zero-mass, and propagate
\[
 \bar\eta_{r-1}=M_r^\top\bar\eta_r+\bar\xi_{r-1}.
\]
Then
\[
\|\bar\eta_0\|_{\TV}\le
\left(\prod_{r=1}^R\tau(M_r)\right)\|\bar\eta_R\|_{\TV}
+\sum_{s=1}^R\left(\prod_{r=1}^{s-1}\tau(M_r)\right)\|\bar\xi_{s-1}\|_{\TV},
\]
where the empty product is one.
\end{theorem}
\begin{proof}
Unroll the recurrence:
\[
 \bar\eta_0=M_1^\top\cdots M_R^\top\bar\eta_R+
 \sum_{s=1}^RM_1^\top\cdots M_{s-1}^\top\bar\xi_{s-1}.
\]
Apply Theorem~\ref{thm:tv} to each product and use the triangle inequality.
\end{proof}
\begin{theorem}[Conditional quotient-recurrence gradient certificate]\label{thm:conditional}
Let $B_T$ map a stopped-base cotangent to its parameter-gradient contribution through the finite base solve. Let $\bar\eta_0$ be the output of the projected recurrence in Theorem~\ref{thm:homogeneous} or~\ref{thm:inhomogeneous}, and define
\[
 \Delta g_R^0:=B_T\bar\eta_0.
\]
Suppose that, for a chosen gradient norm $\|\cdot\|_G$,
\[
 \kappa_T\ge\sup_{\one^\top\xi=0,\,\xi\ne0}
 \frac{\|B_T\xi\|_G}{\|\xi\|_{\TV}}.
\]
In the homogeneous case,
\[
 \|\Delta g_R^0\|_G\le\kappa_T
 \left(\prod_{r=1}^R\tau(M_r)\right)\|\bar\eta_R\|_{\TV}.
\]
In the inhomogeneous case, replace the product term by the right-hand side of Theorem~\ref{thm:inhomogeneous}. If $\eta_0^{\rm raw}$ is instead obtained by propagating unprojected terminal and source cotangents, the full raw contribution obeys only
\[
 \|B_T\eta_0^{\rm raw}\|_G\le
 \|\Delta g_R^0\|_G+\|B_T(\eta_0^{\rm raw}-\bar\eta_0)\|_G.
\]
The second term must vanish by a separate argument or be separately bounded.
In general $\bar\eta_0\ne\Pi_0\eta_0^{\rm raw}$.
\end{theorem}
\begin{proof}
Apply the assumed operator bound for $B_T$ to the zero-mass vector $\bar\eta_0$, followed by Theorem~\ref{thm:homogeneous} or~\ref{thm:inhomogeneous}. The raw comparison follows by adding and subtracting $B_T\bar\eta_0$. Lemma~\ref{lem:projection} forbids replacing the projected recurrence by projection at the end of a raw recurrence without an additional hypothesis.
\end{proof}
\paragraph{Why the recurrence qualifier is necessary.}
For $M=\left(\begin{smallmatrix}0.9&0.1\\0.2&0.8\end{smallmatrix}\right)$ and $\eta_1^{\rm raw}=(1,1)^\top$, the projected recurrence returns $M^\top\Pi_0\eta_1^{\rm raw}=0$, while $\Pi_0M^\top\eta_1^{\rm raw}=(0.1,-0.1)^\top$. Thus the right-hand side of the homogeneous certificate cannot bound an arbitrary raw recurrence projected only at its endpoint. This correction changes the identified certified quantity, not the stochastic contraction inequalities.

\subsubsection{Shared-component overlap and sparse negative cases}
\begin{theorem}[Shared-component Markov overlap]\label{thm:shared}
Let $M$ be row-stochastic. If a probability vector $q$ and $\gamma\in[0,1]$ satisfy $M_{p,\cdot}\ge\gamma q^\top$ componentwise for every row, then $\tau(M)\le1-\gamma$.
\end{theorem}
\begin{proof}
Every pair of rows shares mass at least $\sum_k\gamma q_k=\gamma$. The overlap definition of $\tau$ gives the result.
\end{proof}
A shared dustbin token can create such a component for some supports, but this Markov statement does not itself establish an unbalanced OT model or arbitrary gap capacity; a Sinkhorn-specific application must derive the lower bound from its actual support, marginals, and scores.
\begin{proposition}[Disconnected-support obstruction]\label{prop:disconnected}
If a row-stochastic derivative kernel has two closed disconnected components, then $\tau(M)=1$ and no global oscillation contraction controls modes that separate them.
\end{proposition}
\begin{proof}
Rows chosen from different closed components have disjoint transition support, so their overlap is zero. A vector constant on each component but different between components is preserved by the closed-component dynamics.
\end{proof}
\begin{proposition}[Identity-like sparse obstruction]\label{prop:identity}
The identity matrix on at least two states has $\tau(I)=1$ and preserves every oscillation exactly. Thus sparse support alone does not imply one-step global contraction.
\end{proposition}
\begin{proof}
Distinct identity rows have disjoint singleton supports, and $Ih=h$.
\end{proof}

\subsection{Additional CPU-Scale Validation}
A targeted compute-light rerun on August 11, 2026 reported the following checks. These are synthetic theory validations, distinct from the post-submission TPU and trained-baseline experiments reported earlier in this version. The five directly relevant test files contain 33 tests, all of which passed in the same rerun.
\begin{table}[htbp]\centering\small
\begin{tabular}{@{}p{0.20\linewidth}p{0.36\linewidth}p{0.36\linewidth}@{}}\toprule
Check & What it verifies & Current result\\\midrule
Exact VJP & Final $v^{(R)}$ pullback in finite stopped-tail reverse mode & Terminal-half-step regression: passed\\\addlinespace
Transport orbit & Same-score plan and cotangent reconstruction & Plan error $5.68\cdot10^{-14}$; cotangent error $5.68\cdot10^{-14}$\\\addlinespace
Centered gauge & Corrected and uncorrected mixed-time VJPs & Corrected error $2.98\cdot10^{-8}$; uncorrected error $7.75\cdot10^{-2}$\\\addlinespace
General marginals & Rectangular compatible marginals and finite-tail VJP & Five tests passed; VJP error $2.98\cdot10^{-8}$\\\addlinespace
Dobrushin & Coefficient, primal oscillation, and zero-mass dual TV bounds & Sixteen tests passed\\\addlinespace
Sparse negative case & Disconnected and identity-like kernels & $\tau=1$ in both cases\\\addlinespace
Dustbin overlap & Synthetic shared-component mixing & $\tau=0.75$\\\bottomrule
\end{tabular}
\caption{August 11, 2026 compute-light validation of the additional v3 theory. The rerun used CPU-scale synthetic inputs and required no TPU, Pfam shards, checkpoints, or cloud storage. These are retained historical results, not a new execution of that test suite for this editorial revision.}\label{tab:v3checks}
\end{table}
The exact terminal-half-step regression and the other test and script names are recorded in the accompanying source repository.

\subsection{Edge-Output Complexity}
\begin{proposition}[Edge-linear explicit output]\label{prop:output}
Let $\Omega$ be an arbitrary active support. Any algorithm that explicitly outputs the exact score cotangent $\bar S_e$ for every active edge $e\in\Omega$ requires output time linear in $|\Omega|$. Thus the one-reference orbit representation is an edge-linear streaming representation, not a sub-edge algorithm for explicitly writing $\bar S$.
\end{proposition}
\begin{proof}
The requested object has one scalar per active edge. This is only an output lower bound; it is not a general edge-inspection lower bound for compressed representations or scalar functionals, where cancellations can occur.
\end{proof}
\paragraph{Scope.}
The orbit results assume fixed support, fixed effective score within a temperature stage, and a finite computation graph. The sparse certificates control the declared zero-mass quotient recurrence; contributions from unprojected terminal or source cotangents require a separate argument. The conditional gradient certificate also requires a finite instance- and norm-specific bound on $\kappa_T$. Finally, dustbin overlap is a possible mixing channel, not by itself an unbalanced-transport or arbitrary-gap theorem.

\end{document}